\documentclass[review, 10pt]{elsarticle}

\usepackage{lipsum}
\usepackage{lscape}
\usepackage{graphicx}
\usepackage{caption}
\usepackage{subcaption}
\usepackage{setspace}
\usepackage{amssymb}
\usepackage{amsmath, amsthm, amssymb}
\usepackage{mathtools}
\usepackage{multirow}
\usepackage{afterpage,lscape}
\usepackage{float}
\usepackage{algorithm}
\usepackage{algpseudocode}
\usepackage{xcolor}
\usepackage{hyperref}

\newtheorem{theorem}{Theorem}
\newtheorem{prop}{Proposition}
\newtheorem{lemma}{Lemma}
\newtheorem{definition}{Definition}
\newtheorem{remark}{Remark} 
\begin{document}

\begin{frontmatter}

\title{Skew-Probabilistic Neural Networks for Learning from Imbalanced Data}

\author[1]{Shraddha M. Naik}
\author[2]{Tanujit Chakraborty\footnote{Joint First Author and Corresponding Author: tanujit.chakraborty@sorbonne.ae}}
\author[2]{Madhurima Panja}
\author[2]{Abdenour Hadid}
\author[3,4,5]{Bibhas Chakraborty}
\affiliation[1]{organization={Khalifa University of Science and Technology},
 city={Abu Dhabi},
country={UAE}}
\affiliation[2]{organization={Sorbonne University Abu Dhabi},
            addressline={SAFIR}, 
            city={Abu Dhabi},
            country={UAE}}
\affiliation[3]{organization={Duke-NUS Medical School, National University of Singapore},
            country={Singapore}}
\affiliation[4]{organization={Department of Statistics and Data Science, National University of Singapore},
country={Singapore}}
\affiliation[5]{organization={Department of Biostatistics and Bioinformatics, Duke University},
country={USA}}

\begin{abstract}
Real-world datasets often exhibit imbalanced data distribution, where certain class levels are severely underrepresented. In such cases, traditional pattern classifiers have shown a bias towards the majority class, impeding accurate predictions for the minority class. This paper introduces an imbalanced data-oriented classifier using probabilistic neural networks (PNN) with a skew-normal kernel function to address this major challenge. PNN is known for providing probabilistic outputs, enabling quantification of prediction confidence, interpretability, and the ability to handle limited data. By leveraging the skew-normal distribution, which offers increased flexibility, particularly for imbalanced and non-symmetric data, our proposed Skew-Probabilistic Neural Networks (SkewPNN) can better represent underlying class densities. Hyperparameter fine-tuning is imperative to optimize the performance of the proposed approach on imbalanced datasets. To this end, we employ a population-based heuristic algorithm, the Bat optimization algorithm, to explore the hyperparameter space effectively. We also prove the statistical consistency of the density estimates, suggesting that the true distribution will be approached smoothly as the sample size increases. Theoretical analysis of the computational complexity of the proposed SkewPNN and BA-SkewPNN is also provided. Numerical simulations have been conducted on different synthetic datasets, comparing various benchmark-imbalanced learners. Real-data analysis on several datasets shows that SkewPNN and BA-SkewPNN substantially outperform most state-of-the-art machine-learning methods for both balanced and imbalanced datasets (binary and multi-class categories) in most experimental settings.
\end{abstract}


\end{frontmatter}

\section{Introduction}\label{sec1}
Data imbalance is ubiquitous and inherent in real-world applications, encompassing rare event prediction with applications in medical diagnosis, image classification, customer churn prediction, and fraud detection. Instead of preserving an ideal uniform distribution over each class level, real data often exhibit a skewed distribution~\cite{he2009learning, krawczyk2016learning} where specific target value has significantly fewer observations. This situation is known as the ``curse of imbalanced data'' in the pattern recognition literature~\cite{fernandez2018learning}. {\color{black}To elaborate more, in applications such as software defect prediction \cite{chakraborty2020hellinger}, credit card fraud detection \cite{chakraborty2020superensemble}, medical diagnosis \cite{yuan2024clinical}, image segmentation \cite{ghosh2024multi}, etc., events corresponding to different classes are observed with unequal probabilities. For example, classifying between defect-prone and non-defect-prone categories in software modules is of great interest to software engineers, playing a vital role in reducing software development costs and maintaining high quality. The defective class is a minority class, although smaller in number, but it carries more cost when misclassified by the software defect prediction models. Similarly, any credit card provider may only observe fewer fraudulent cases among millions of daily transactions. In medical diagnosis, a disease prediction system may find a very limited number of malignant tumor cases compared to many benign cases. All the examples mentioned earlier result in imbalanced data settings where the training set contains an unequal number of instances in the class labels. The imbalanced nature of real-world datasets poses} challenges to conventional classifiers, as they tend to favor the majority class {\color{black}(abundance of training instances)}, resulting in higher misclassification rates for minority class {\color{black}(scarcity of training examples)} samples. Therefore, diverse approaches have been introduced to address class imbalance. Existing solutions for learning from imbalanced data can be categorized into data-level and algorithmic-level approaches~\cite{krawczyk2016learning}.\\

\noindent Data-level approaches primarily focus on balancing the distribution of samples between the majority and minority classes, ultimately creating a harmonized dataset for learning endeavors~\cite{baesens2021robrose,gu2022novel}. Data-level approaches can be broadly categorized into three primary sub-categories: under-sampling, over-sampling, and mixed-sampling. Several under-sampling techniques have been introduced, including modified versions of random sampling~\cite{tahir2012inverse, wang2023kernel}, clustering-based approaches~\cite{yen2009cluster, farshidvard2023novel}, and evolutionary algorithm-based approaches~\cite{zheng2021automatic, gong2022hybrid}. Among these, clustering-based methods (e.g., edited nearest neighbor) share a common approach of grouping samples with similar attributes and selecting a specific number from each cluster based on predefined rules. On the contrary, evolutionary approaches view under-sampling as an optimization challenge within swarm intelligence search algorithms to minimize information loss and mitigate prediction bias. Moving on to oversampling methods, these techniques typically involve duplicating existing minority samples or generating new ones based on specific rules to balance the number of majority and minority instances. A widely recognized approach in this category is the synthetic minority oversampling technique (SMOTE)~\cite{chawla2003smoteboost, fernandez2018smote}. SMOTE achieves this by generating new minority instances close to other existing minority examples through interpolation. However, it is worth noting that SMOTE does not consider the data's density, potentially leading to increased overlap between classes. As a response to this limitation, several extensions of SMOTE have been proposed to address this issue, including majority-weighted minority oversampling~\cite{barua2012mwmote}, borderline-SMOTE~\cite{han2005borderline}, adaptive synthetic sampling (ADASYN)~\cite{he2008adasyn}, convex hull based SMOTE~\cite{yuan2023chsmote}, and SMOTE with boosting (SWB) \cite{sauglam2022novel}. The mixed-sampling method integrates both under-sampling and over-sampling techniques, effectively mitigating the drawbacks of reduced diversity in the minority class and substantial information loss in the majority class~\cite{xu2021improved, koziarski2021csmoute}. {\color{black}Alongside the sampling techniques, optimized decision threshold-based frameworks like the Generalized tHreshOld ShifTing (GHOST) framework have been recently introduced for handling class-imbalance problems \cite{esposito2021ghost}.} It is worth noting that Generative Adversarial Networks (GANs) also possess a high capability to generate synthetic data and to enhance the performance of data-driven machine learning models, especially with deep learning~\cite{chakraborty2023years}. However, data-level approaches are very sensitive to the presence of outliers, and sampling strategies excessively distort the data distribution, often resulting in worse performance in real-world scenarios~\cite{elor2022smote}.\\

\noindent Algorithm-level techniques encompass the adaptation of established classifiers to enhance their predictive capabilities, particularly with regard to the minority class. Within this domain, an array of strategies has emerged, leveraging well-known classifiers, including support vector machines (SVM)~\cite{akbani2004applying}, nearest neighbor methods~\cite{cano2013weighted,he2009learning}, as well as decision trees and random forests~\cite{sardari2017hesitant, chakraborty2020superensemble}. 
Among the algorithmic-level methods for imbalanced classification, the Hellinger distance decision tree (HDDT), which employs Hellinger distance as the splitting criterion, is a prominent technique in the literature~\cite{cieslak2008learning}. HDDT, alongside its ensemble counterpart Hellinger distance random forest (HDRF)~\cite{daniels2017addressing} and other extensions like {\color{black}inter-node HDDT (iHDDT)} \cite{akash2019inter}, adeptly handles both balanced and imbalanced datasets. However, HDDTs, while robust against class imbalance, may face overfitting problems due to a lack of pruning~\cite{boonchuay2017decision}. They can also encounter challenges with sticking to local minima and overfitting, particularly with large trees~\cite{chaabane2020enhancing}. To avoid these challenges, Hellinger net (HNet) \cite{chakraborty2020hellinger} was proposed to tackle the data imbalance challenge for software defect prediction problems. However, the computational expenses resulting from HNet training pose a real challenge when working with tabular datasets. Also, there is a lack of work on developing neural network-based algorithmic solutions dealing with tabular datasets in the imbalanced domain.  \\

\noindent Artificial neural networks (ANNs) have revolutionized machine learning, providing robust solutions for various tasks, including classification and regression~\cite{duda1995pattern, abiodun2019comprehensive}. Among the various types of ANNs, probabilistic neural networks (PNN) have gained significant attention due to their distinctive capability of providing probabilistic outputs~\cite{specht1990probabilistic1}. This unique feature allows PNN to perform accurate predictions and quantify the uncertainty associated with each prediction, making them highly valuable in dealing with uncertain real-world scenarios. At the heart of PNN lies the utilization of a Gaussian kernel density estimation to model the feature vectors of each class in the training data~\cite{mao2000probabilistic}. This non-parametric kernel-based approach enables the PNN to estimate the probability density function for each class, thereby facilitating accurate classification. However, despite its effectiveness in certain scenarios, the Gaussian kernel may not always be optimal for real-world imbalanced datasets due to their complexities and uncertainties, which a single Gaussian model cannot fully capture. In the classical statistics literature, skew-normal distribution is popular for modeling data with non-zero skewness and asymmetric characteristics~\cite{gupta2004characterization, azzalini2005skew}. The current work introduces a probabilistic neural network based on the skew-normal kernel function, which we call SkewPNN. Our approach estimates multiple smoothing parameters through a heuristic Bat algorithm. This method is computationally inexpensive and effectively handles benchmark tabular datasets with varied imbalanced settings. \\

\noindent Furthermore, we theoretically show the consistency of the density estimates in SkewPNN so that it can classify data patterns with unknown classes based on the initial set of patterns, the real class of which is known. By shedding light on the benefits of this flexible and adaptable approach, our work lays the foundation for improved decision-making and problem-solving across a wide range of practical applications (balanced and imbalanced data). {\color{black}In essence, our contributions are as follows:
\begin{enumerate}
    \item We propose two solution strategies, namely SkewPNN and Bat algorithm-based SkewPNN (BA-SkewPNN), utilizing the skew-normal kernel function. Our proposal improves the performance mainly for imbalanced data but also works well for balanced data classification.
    \item BA-SkewPNN utilizes a population-based Bat algorithm to optimize the skew-normal kernel function parameters. This strengthens SkewPNN's ability to handle highly imbalanced data scenarios. A theoretical analysis of the computational complexities of the proposal is also given. 
    \item We establish theoretical results for the consistency of the density estimates for our proposed method, ensuring that the true distribution will be approached smoothly. Our proposed methods generalize probabilistic neural networks; therefore, they can accurately predict target probability scores. 
    \item One of the key advantages of our methods is interpretability since it leverages statistical principles to perform classification tasks. The decision boundaries generated by SkewPNN are smoother and more flexible than those generated by the decision tree and random forest-based solution for imbalanced data problems.
    \item Experimental results on 20 real imbalanced and 15 real balanced datasets (binary and multiclass) show that the proposed methods are more suitable for imbalanced learning problems than other competitive methods, which is further confirmed by the results of the statistical significance tests.
\end{enumerate}}

\noindent The rest of the paper is structured as follows. Section \ref{sec2} provides a comprehensive background on skew-normal distribution, PNN, and the Bat optimization algorithm. In Section~\ref{sec3}, the design of the proposed method is presented, along with the introduction of the fitness function used in the Bat algorithm and its theoretical properties. Experimental results and performance analysis are presented in Section~\ref{sec4}. Finally, we conclude the paper with a discussion of the findings and future scope for further research in Section~\ref{sec5}. Section \ref{appendix} Appendix gives an illustrative example with numerical values.

\section{Preliminaries}\label{sec2}
This section offers a glimpse into the component methods to be used as building blocks for the proposed methods. We discuss their mathematical formulations and relevance in this study. 
\subsection{Skew-Normal Distribution}

In the last few decades, the skew-normal distribution (both univariate and multivariate)~\cite{azzalini1985class, azzalini1996multivariate, azzalini1999statistical, azzalini2005skew} has received great importance for modeling skewed data in statistics, econometrics, nonlinear time series, and finance among many others. The following lemma (for proof, see~\cite{azzalini2005skew}) played a crucial role in its development:
\begin{lemma}
    Let $f_0(\cdot)$ be a $d$-dimensional continuous probability density function (PDF) symmetric around $0$. Let $G$ be a one-dimensional cumulative distribution function such that $G'$ (first derivative of $G$) exists and $G'$ is a density symmetric around $0$. Then, $f(x)$ is a density function on $\mathbb{R}^d$ for any (real-valued) odd function $h(x)$ as defined below:
    $$ f(x) = 2 f_0 (x) G\left\{ h(x)\right\}; \quad x \in \mathbb{R}^d.$$
\end{lemma}
\noindent The above lemma suggests that a symmetric `basis' density $f_0$ can be manipulated via a perturbation function $G\{h(x)\}$ to get a skewed density $f(x)$. Also, the perturbed density will always include a `basis' density. On using $f_0 = \phi(\cdot)$ (density function of the standard normal distribution) and $G = \Phi(\cdot)$ (distribution function of the standard normal distribution) and $h(x) = \alpha x$, where $\alpha \in \mathbb{R}$, we get the skew-normal distribution with shape parameter $\alpha$, denoted by $\operatorname{SN}(\alpha)$, and defined as:
\begin{equation}\label{Eq_1}
    f(x, \alpha) = 2 \phi(x) \Phi(\alpha x); \quad -\infty < x, \; \alpha < \infty.    
\end{equation}
For $\alpha = 0$, we obtain the standard normal density. A more general form (also called a three-parameter skew-normal distribution) has received considerable attention in the last two decades because of its greater flexibility and applications in various applied fields and is as follows:
\begin{equation*}
    f(x; \xi, \sigma, \alpha) = \frac{2}{\sigma} \phi\left(\frac{{x - \xi}}{\sigma} \right) \Phi\left(\frac{\alpha (x - \xi) }{\sigma} \right); \quad -\infty < x, \; \xi, \; \alpha < \infty, \; \sigma > 0,   
\end{equation*}
where $\xi$ is the location parameter, $\sigma$ is the scale parameter and $\alpha$ is a skewness (also known as tilt) parameter. {\color{black}Note that, throughout this manuscript, we represented the scale parameter of both normal and skew-normal distribution by $\sigma$ to avoid confusion; however, in the experimental setting, the values of $\sigma$ in normal and skew-normal may be different.} Several authors have explored this model to analyze skewed and heavy tail data due to its flexibility~\cite{gupta2004multivariate, kundu2014geometric, arnold2004characterizations, genton2005discussion}. A multivariate extension is straightforward~\cite{azzalini1996multivariate} and can be defined as follows: A random vector $\mathbf{X} = \left(X_1, X_2, \ldots, X_d\right)^T$ is said to follow a $d$-dimensional multivariate skew-normal distribution if it has the following density function:
$$
    f_{d}\left(\mathbf{x}\right) = 2 \phi_d \left(\mathbf{x}, \mathbf{\Sigma}\right) \Phi\left(\mathbf{\alpha}^T \mathbf{x} \right); \quad \mathbf{x} \in \mathbb{R}^d,
$$
where $\phi_d \left(\mathbf{x}, \mathbf{\Sigma}\right)$ denotes the PDF of $d$-dimensional multivariate normal with correlation matrix $\mathbf{\Sigma}$ with $\mathbf{\alpha}$ is the shape parameter. Inferential statistical properties have been studied, and their successful applications are found to analyze several multivariate datasets in various fields (for e.g., reliability and survival analysis~\cite{kundu2014geometric}) due to their flexibility.

\subsection{Probablistic Neural Networks (PNN)}
Probabilistic neural networks (PNN) are a specialized type of single-pass neural network known for their unique architecture. Unlike multi-pass networks, PNN does not rely on iterative weight adjustments, simplifying their operation and making them well-suited for real-time applications~\cite{specht1990probabilistic}. PNN are particularly notable for their use of Bayesian rules, which enable them to estimate posterior probabilities~\cite{specht1990probabilistic, richard1991neural}.
As a powerful tool for pattern classification tasks, PNN has garnered significant attention for their capacity to handle uncertainty and offer valuable insights in various real-world applications. This architecture is characterized by four layers, such as an input layer for processing test patterns, a pattern layer with neurons corresponding to training patterns, a summation layer representing class neurons, and an output layer providing the final classification~\cite{specht1990probabilistic}. PNN has demonstrated their value across various fields, including medicine, science, business, and industries, owing to their ability to handle real-time data effectively~\cite{zhang2000neural}. 

{\color{black}During training, PNN utilizes the Parzen window, a non-parametric density estimation technique, to approximate the underlying data distribution. The working principles of PNN are as follows: (1) Given a new input (test point), the pattern layer computes the similarity between the input and available training samples; (2) The summation layer estimates the probability of the input belonging to each class; and (3) The output layer selects the class with the highest probability.} Therefore, for the prediction phase, a winner-takes-all strategy is employed, where the class with the highest probability is selected as the output. This capacity to model the entire PDF empowers PNN to effectively handle both binary and multi-class classification tasks with remarkable accuracy~\cite{mao2000probabilistic}. {\color{black}The training process in PNN is fast as they only need one pass through the training data.} In PNN, the Gaussian distribution is favored for its well-understood properties and straightforward mathematical form, characterized by mean ($\mu$) and standard deviation ($\sigma$) which defines the shape of the distribution. Within PNN, it is used to estimate class probabilities, associating each training pattern with a Gaussian {\color{black}kernel} function centered on its feature values. The spread parameter, determining function width, significantly influences smoothness and generalization. PNN's effective modeling of data distribution using Gaussian kernel makes them suitable for tasks like classification and regression, offering advantages in training, interpretability, and multi-class handling.

Conventional PNN employs the following mathematical formulation to calculate the PDF of a set of random variables $X_1, X_2, \ldots, X_n$ with unknown PDF $g(x)$ is estimated using a family of estimators for each class~\cite{specht1990probabilistic} as 
\begin{equation}\label{Eq2}
   f_{n}(x) = \frac{1}{{nh(n)}} \sum_{i=1}^{n} K\left(\frac{{x - x_i}}{{h(n)}}\right),
\end{equation}
where $K(\cdot)$ is a kernel function (discussed below) and the bandwidth $h(n)$ constitutes a sequence of numbers satisfying the condition 
\begin{equation}\label{Eq_9}
  \lim_{{n \to \infty}} h(n) = 0.  
\end{equation}
This process entails the utilization of a non-parametric density estimation technique, such as the Parzen window, to approximate the underlying data distribution for each class ($c$) individually and given as
\begin{equation} 
   P_c(x) = \frac{{f_n(x)}}{{\sum {f_n(x)}}}.
\end{equation}
In vector formulation, for a given new input sample $x$ and a training sample $x_i$ from class $c$, the Gaussian kernel can be represented as
\begin{equation}\label{Eq_Gaussian_Kernel}
{\color{black}K(x,x_i) = \exp\left(-\frac{\left\|x - x_i\right\|^2}{2\sigma^2}\right),}
\end{equation}
where $\sigma$ represents the smoothing parameter ({\color{black}controls the width of the Gaussian kernel}), $x$ is the vector representing the input sample, and $x_i$ is the vector representing the training sample from class $c$. The Gaussian kernel computes the similarity between the new input sample and each training sample, with higher values indicating stronger similarity. It represents a symmetric, bell-shaped curve centered at $0$, with the standard deviation determining the width of the curve.
Several notable extensions of PNN were introduced in the recent literature, namely weighted PNN~\cite{montana1992weighted}, self-adaptive PNN~\cite{yi2016improved}, and Bat-algorithm-based PNN~\cite{naik2020bat}, among many others which addresses the complexities of parameter estimation in PNN. However, none of the above provides a robust solution to the data imbalance problem.

\subsection{Bat Algorithm (BA)}\label{Sec_BA_Algo}
Bat algorithm (BA) is a meta-heuristic optimization technique inspired by the echolocation behavior observed in Bats during their foraging activities~\cite{yang2010new}. It commences by randomly placing a population of $B$ Bats within the search space, with each Bat symbolizing a potential solution to the optimization problem at hand~\cite{yang2019ba,naik2020bat}.

At every iteration $t$, the velocity $(v_i(t))$ of each Bat $i$ is updated based on the disparity between its current position $(\varkappa_i(t))$ and the best solution discovered $(\varkappa_{\text{best}})$. The velocity adjustment is regulated by a fixed frequency $f_{min}$ with a maximum limit of $f_{max}$ and dynamic frequency $f_i$. The velocity $v_i(t+1)$ at time step $t+1$ is determined as
\begin{equation}
v_i(t+1) = v_i(t) + (\varkappa_{\text{best}} - \varkappa_i(t)) f_i,
\end{equation}
where $f_i$ is computed as $f_i = f_{min} + (f_{max}-f_{min}) \beta$, with $\beta \in [0,1]$. Following the velocity update, the new position $(\varkappa_i(t+1))$ of each Bat $i$ is computed by adding the updated velocity to its current position
\begin{equation}
\varkappa_i(t+1) = \varkappa_i(t) + v_i(t+1).
\end{equation}
Certain Bats may engage in local search by exploring neighboring regions. This is achieved by introducing a random displacement controlled by a parameter $\epsilon$, defined as
\begin{equation}
\varkappa_i(t+1) = \varkappa_i(t) + \epsilon A_t,
\end{equation}
where $A_t$ represents the average loudness, and $\epsilon \in [-1,1]$. Initially, the positive loudness $A_0$ is set. To adapt to evolving conditions during the optimization process, the loudness $(A_i(t+1))$ of each Bat $i$ diminishes over time to decrease the appeal of a solution. This is updated using a scaling factor $\lambda$ as $A_i(t+1) = \lambda A_i(t)$. Similarly, the pulse rate $(r_i(t+1))$ governs the likelihood of emitting ultrasound pulses and is adjusted throughout the optimization process. BA iterates until a specified stopping condition is fulfilled, which could be reaching a maximum iteration count or attaining a predefined level of precision. This iterative process enables the algorithm to navigate the solution space effectively and ultimately converge toward optimal or highly satisfactory solutions. This versatility and effectiveness render the BA as a valuable and potent tool across a range of problem domains, and this nature-inspired algorithm combined with PNN having a skew-normal kernel is used in this study to address the class imbalance problem. 

\section{Proposed Method}\label{sec3}
This section introduces the proposed SkewPNN, followed by Bat algorithm-based SkewPNN and BA-SkewPNN. Furthermore, we study the consistency of density estimates in SkewPNN. 

\subsection{Algorithm 1: SkewPNN}\label{algo_skewpnn}
{\color{black}SkewPNN consists of four layers structured as follows:
\begin{itemize}
    \item Input Layer: Accept feature vectors of the data.
    \item Pattern layer: Computes distances between the input vector and each training sample using a similarity measure (skew-normal kernel as defined in Eqn. \ref{eq_skew_normal_kernel}).
    \item Summation Layer: Aggregates the outputs of the pattern layer for each class.
    \item Output Layer: Assign the input to the class with the highest posterior probability.
\end{itemize}}

\noindent To elaborate, the skew-normal distribution is used as a kernel function in PNN for situations where the data exhibits skewed distributions with a long tail. The skew-normal density, defined by its location parameter $\xi$, scale parameter $\sigma$, and skewness parameter $\alpha$, is as follows:
\begin{equation}\label{Eq_10}
    f(x; \xi, \sigma, \alpha)  = \frac{2}{{\sigma\sqrt{2\pi}}} \exp\left(-\frac{{(x-\xi)^2}}{{2\sigma^2}}\right) \Phi \left\{\alpha \left(\frac{x-\xi}{\sigma}\right)\right\}; \quad \xi, x \in (- \infty, \infty), \; \sigma > 0,
\end{equation}    
{\color{black} where $\xi$ determines the central point of the distribution, $\sigma$ controls the spread of the distribution, $\alpha$ introduces asymmetry (shape parameter), and $\Phi\left(z\right)$ denotes the standard normal cumulative distribution function. For $\alpha = 0, \; f(x)$ is simply the normal distribution where $\alpha > 0$ denotes a longer tail to the right and $\alpha < 0$ denotes the longer tail to the left.} The use of skew-normal kernel function will allow the PNN to effectively handle complex and non-symmetric data distributions, making it a valuable choice for various real-world applications. By considering the skewness of the data, the skew-normal kernel enhances the PNN's ability to provide accurate and probabilistic outputs for classification tasks, particularly when dealing with uncertain and imbalanced datasets. The skew-normal kernel function (without a location parameter) is defined as:
\begin{equation} 
    {\color{black}K\left(x, x_i\right) = \exp\left(-\frac{\left\|x - x_i\right\|^2}{2\sigma^2}\right) \Phi\left(\frac{\alpha  \left\|x - x_i\right\|}{\sigma}\right),}
    \label{eq_skew_normal_kernel}
\end{equation}
where $x$ is the vector representing the new input sample, $x_i$ is the vector representing the training sample from class $c$, $\sigma$ represents the smoothing parameter, and $\Phi$ denotes the cumulative distribution function (CDF) of the standard normal distribution. The skew-normal kernel function now properly accounts for the smoothing parameter $\sigma$ in the cumulative distribution term, making it a suitable alternative to the Gaussian kernel {\color{black}(as in Eqn. \ref{Eq_Gaussian_Kernel})} for handling non-symmetric data distributions in PNN. Incorporating skew-normal kernels introduces asymmetry to the distribution, enabling a more flexible data representation. This plays a pivotal role in PDF computations during both the training and prediction stages of the SkewPNN model, with the selection of the skewness parameter being contingent on the dataset's characteristics and the specific problem being addressed. {\color{black}In the summation layer of SkewPNN, for a given input $x$, the class probabilities (for class $c$) are computed as:
\begin{equation*}
    P_c\left(x\right) = \frac{1}{N_c} \sum_{x_i \in \text{ class } c} K\left(x, x_i\right),
\end{equation*}
where $N_c$ is the number of training samples in class $c$. The decision in the output layer is the predicted class $\hat{c}$, which is the one with the highest posterior probability:
\begin{equation*}
    \hat{c} = \underset{c}{\operatorname{arg max}} P_c\left(x\right).
\end{equation*}
The key difference between standard PNN and SkewPNN lies in the pattern layer activation function, which uses the skew-normal kernel instead of the Gaussian kernel. The hyperparameters in SkewPNN are $\sigma$ (smoothing parameter) and $\alpha$ (skewness parameter), and adjusting $\alpha$ helps in handling data imbalance problems.} Since the Gaussian kernel is a particular case of Eqn. (\ref{Eq_10}) with constant terms, SkewPNN can also be useful for classifying balanced datasets. {\color{black}An illustrative example with numerical values is given in Appendix (Section \ref{appendix}) to explain how the mechanism works in handling imbalanced data.}

\begin{remark}
    {\color{black}The proposed SkewPNN model addresses class imbalance by using the skew-normal kernel, which introduces a skewness parameter $(\alpha)$ to model asymmetrical data distributions adaptively. This parameter enables the kernel to amplify the contribution of minority class samples, ensuring balanced probability estimations across classes. Theoretical properties of the skew-normal distribution ensure that minority class densities are neither underestimated nor dominated by majority class contributions. By naturally reshaping the density estimation process, the SkewPNN framework balances soft probabilistic predictions, improving classification performance in imbalanced settings. In traditional PNN, minority samples contribute minimally to overall probability scores, whereas SkewPNN’s kernels amplify the influence of minority samples due to their alignment with the underlying skewed distribution.}
\end{remark}


\subsection{Algorithm 2: BA-SkewPNN}
For optimizing the parameters of the skew-normal kernel (smoothing parameter $\sigma$ and skewness parameter $\alpha$) in SkewPNN, the population-based Bat algorithm is employed. Using the Bat optimization approach, the SkewPNN can fine-tune its parameters effectively and achieve improved performance, especially when dealing with imbalanced datasets and data with asymmetric target distributions. The Bat algorithm's ability to explore the hyper-parameter space and find suitable solutions enhances the SkewPNN's capability to handle uncertainty and class imbalance, leading to better classification accuracy and robustness in practical applications. The fitness function in the context of the BA is a crucial component that evaluates the quality of candidate solutions (Bats) during the optimization process. The primary objective of the fitness function is to quantify how well each Bat performs in relation to the optimization problem at hand. It provides a numerical measure of the fitness or suitability of a solution, enabling the algorithm to distinguish between better and worse solutions. In the case of optimizing the parameters of SkewPNN, the fitness function is designed to assess the performance of the SkewPNN on the given classification task. It typically uses performance metrics such as accuracy, F1-score, or area under the receiver operating characteristic curve (AUC-ROC) to measure how well the SkewPNN classifies the data. During each iteration of the Bat Algorithm, the fitness function is applied to evaluate the performance of each Bat's solution. The algorithm then uses this fitness value to update the Bats' positions, velocities, loudness, and pulse rates, guiding the search towards better solutions. The goal of the BA is to find the optimal or near-optimal set of hyper-parameters that maximizes the performance of the SkewPNN on the given classification task. By iteratively evaluating and updating the fitness of the candidate solutions, the BA efficiently searches the solution space, eventually converging to a set of hyper-parameters that yields improved classification accuracy and robustness, especially on imbalanced datasets. The fitness function's design and choice of performance metrics are critical considerations, as they directly influence the success of the optimization process and the final performance of the SkewPNN on the classification task. {\color{black}Given the fitness function, the steps of integrating BA with SkewPNN are as follows:
\begin{enumerate}
    \item Initialize Population: Randomly initialize $B$ Bats with values for $\sigma$ and $\alpha$ with predefined ranges 
    $$ \sigma \in \left[\sigma_{\operatorname{min}}, \sigma_{\operatorname{max}}\right], \; \alpha \in \left[\alpha_{\operatorname{min}}, \alpha_{\operatorname{max}}\right].$$
    \item Update Position and Velocities: Update the positive $\varkappa_{i}\left(t\right)$ and velocity $v_{i}\left(t\right)$ of each Bat at iteration:
    \begin{align*}
        f_i &= f_{\operatorname{min}} + \left(f_{\operatorname{max}} - f_{\operatorname{min}}\right)\beta \\
        v_{i}\left(t+1\right) & = v_{i}\left(t\right) + \left(\varkappa_{\operatorname{best}} - \varkappa_{i}\left(t\right)\right) f_i, \\
        \varkappa_{i}\left(t + 1\right) & = \varkappa_{i}\left(t\right) + v_{i}\left(t+1\right),
    \end{align*}
    where $f_i$ and $\beta$ are defined in Section \ref{Sec_BA_Algo}.
    \item Perform Local Search: If a Bat is selected, perform a random walk:
    $$\varkappa_{\operatorname{new}} = \varkappa_{\operatorname{current}} + \epsilon A_t,$$
    where $\epsilon \in [-1, 1]$ and $A_t$ is loudness at iteration $t$.
    \item Evaluate Fitness and Update $A_t$: Use SkewPNN to classify a validation dataset using the current Bat's hyperparameters to compute the fitness score:
    $$A_i\left(t+1\right) = \lambda A_i\left(t\right),$$
    where $\lambda$ is the scaling factor. Using pulse rate, we let the Bats converge. Based on the highest fitness score, we identify the appropriate Bat, update $\varkappa_{\operatorname{best}}$, and use it in the next iteration.
    \item Stopping Criterion: We stop the algorithm when a maximum number of iterations or a satisfactory fitness score is reached.
\end{enumerate}
Thus, BA-SkewPNN is well-suited for hyperparameter tuning and handles nonlinear hyperparameter space. This results in the `best' values of $\sigma$ and $\alpha$ that optimize the fitness function to improve the performance of the SkewPNN algorithm.} The structure of the proposed models is illustrated in Figs. \ref{fig:Skewpnn_archi} and \ref{fig:BAskewpnn_archi}.
{\color{black}\begin{remark}
\noindent\textbf{Advantages of our proposed methods over traditional PNN and BA-PNN:}
SkewPNN and BA-SkewPNN algorithms modify the kernel used in the pattern layer of PNN. Instead of using a symmetric Gaussian kernel, a skew-normal kernel is applied to (a) model asymmetric distributions and (b) provide greater flexibility in capturing patterns in complex datasets. By tweaking $\alpha$ and $\sigma$, SkewPNN can enhance the influence of the underrepresented classes in the imbalanced datasets. Like PNN, SkewPNN retains its probabilistic framework, making it interpretable and robust for decision-making. In BA-SkewPNN, since we use the bio-inspired optimization algorithm for parameter estimation, this may result in improved accuracy as compared to SkewPNN for real-data applications. 
\end{remark}}

{\color{black}\begin{remark}\label{rem3}
\noindent\textbf{Limitations of the proposed methods:}
Like PNN, SkewPNN presented in this study struggles with large datasets due to its reliance on storing all training samples. SkewPNN requires more memory than some of the statistical classifiers during inference because all training samples are stored in it. SkewPNN may also struggle with high-dimensional datasets. On the other hand, for very small datasets, the added flexibility of the skewness parameter in BA-SkewPNN can sometimes lead to an overfitting problem because of the usage of BA. This usually happens when we run BA for too many iterations, increasing the likelihood of overfitting by narrowing the search space too aggressively. However, we prevent this in BA-SkewPNN by using k-fold cross-validation and early stopping criteria to prevent excessive tuning. 
\end{remark}}

\subsection{Implementations of Algorithm 1 and 2}
A rigorous evaluation was conducted employing the 10-fold cross-validation technique to gauge the models' efficacy. This approach partitions the dataset into ten subsets, utilizing nine subsets for training and one for validation in each iteration. Such an approach offers a robust assessment by reducing the risk of overfitting and enhancing generalization because the model is trained and validated across diverse data segments, maximizing the utilization of available information. The training and testing datasets undergo normalization through z-score normalization. Subsequently, SkewPNN is executed, with the notable modification of substituting the Gaussian kernel with the skew-normal kernel within the proposed framework. The estimation of hyper-parameters is facilitated by the BA, which integrates a unique fitness function encompassing the maximization of the summation of accuracy, AUC-ROC, and F1-score. We delved into diverse configurations of the model's pattern layer. Within this stratum, the PDF of each class is computed to ascertain the probability of an input instance's association with each class. For every PDF, hyper-parameters are derived via two distinct methodologies, applied to both the Gaussian and the skew-normal kernels. In the former approach, the hyper-parameters maintain uniformity across all patterns in the pattern layer. In the latter approach, however, the hyper-parameters vary for each pattern. In this context, using a population-based heuristic algorithm significantly aids in calculating distinct values for different patterns. This is because such algorithms, like the Bat optimization technique used in our study, can systematically explore the parameter space, adapting to the unique characteristics of each pattern. By allowing hyper-parameters to vary individually, the algorithm optimizes their values in alignment with the specific requirements and complexities of each pattern, ultimately enhancing the model's adaptability and performance.


\begin{figure}
    \centering
    \includegraphics[width=0.8\linewidth]{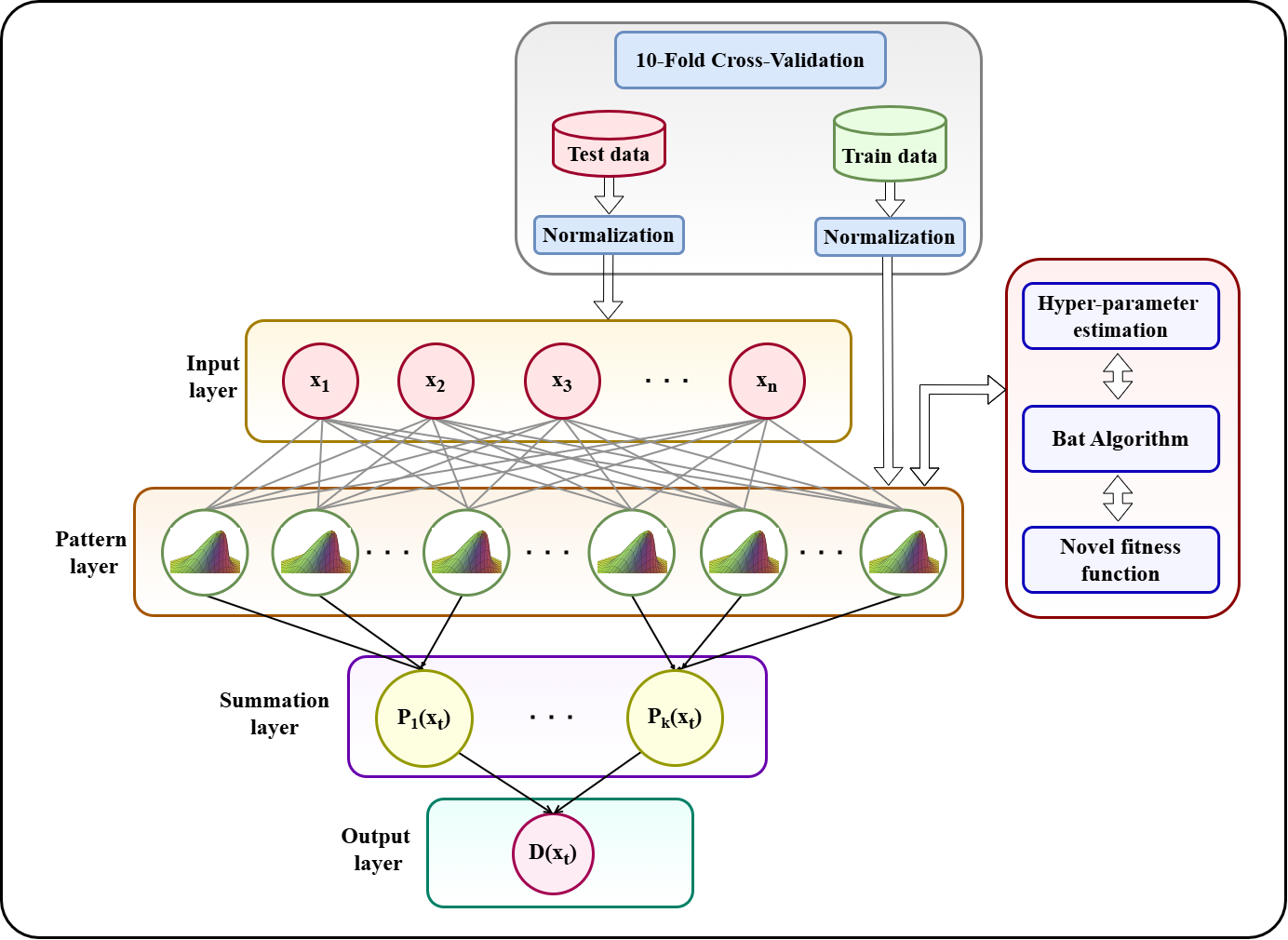}
    \caption{Schematic representation of the proposed SkewPNN and BA-SkewPNN architectures.}
    \label{fig:Skewpnn_archi}
\end{figure}

\begin{figure}
    \centering
    \includegraphics[width=0.8\linewidth]{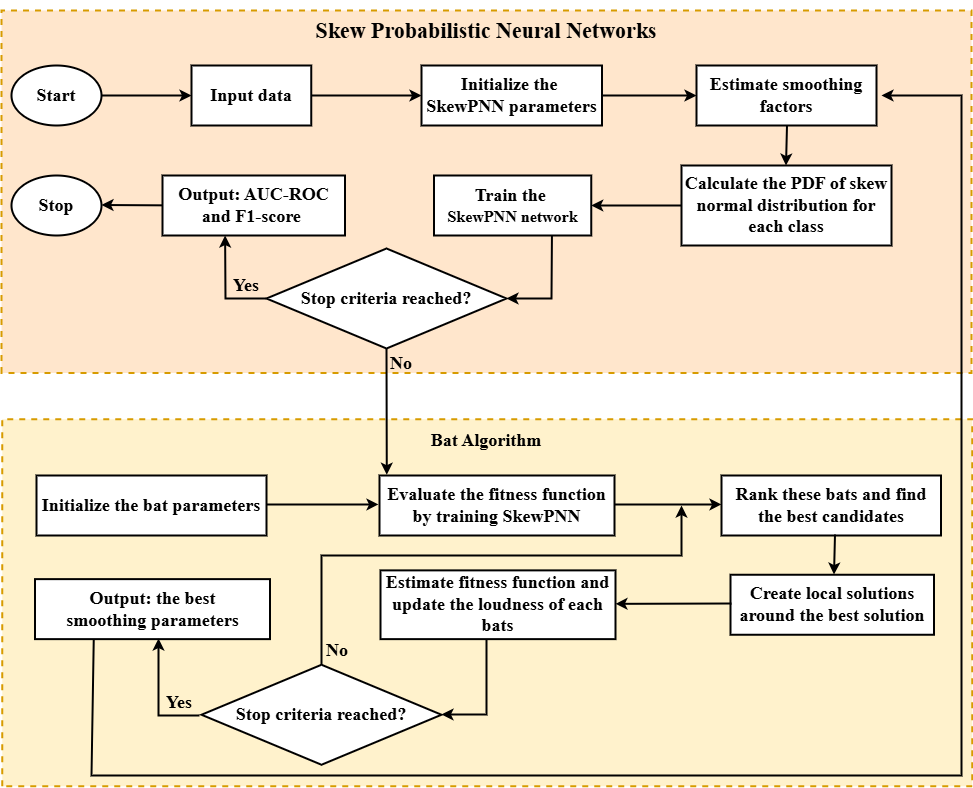}
    \caption{Workflow diagram of the proposed SkewPNN and BA-SkewPNN frameworks.}
    \label{fig:BAskewpnn_archi}
\end{figure}



{\color{black}\subsection{Computational Complexity}
Computational complexity is a crucial indicator for measuring the
efficiency of algorithms. Following the complexity analysis of several imbalanced classifiers done in the past literature \cite{li2024complemented, li2024imbalanced, li2024density}, we calculate the computational complexity of the proposed models. SkewPNN consists of four layers, just like PNN. The attributes of an input vector $\mathbf{x}$ of dimension $d$ comprise the first input layer of the model. The second layer is the pattern layer, composed of as many neurons as training data. This layer applies the kernel to the input defined in Eqn. (\ref{eq_skew_normal_kernel}). The output of the pattern neurons is fed forward to the summation layer that actually gets the average of the output of the pattern units for each class. Suppose there are $J$ neurons in the summation layer, with each $j^{th}$ node acquiring signals from the neurons of the $j^{th}$ class. Finally, the decision layer declares the class assigned to the input vector based on the unit with maximum output from the summation layer. Now, we analyze the computational complexity of the proposed method and compare it with that of PNN and ANN models. Before starting that, we define the Big-$\mathcal{O}$ notation used to describe an algorithm's performance complexity (describing the worst-case scenarios in terms of time or space complexity). Although Big-$\mathcal{O}$ is used to compare the efficiency of various algorithms, they only describe these asymptotic behaviors (not the exact value). 
\begin{definition}
    Given two functions $f\left(n\right)$ and $q\left(n\right)$, we say that $f\left(n\right)$ is $\mathcal{O}\left(q\left(n\right)\right)$ if there exist constants $e > 0$ and $n_0 \geq 0$ such that $f\left(n\right) \leq e q\left(n\right)$ for all $n \geq n_0$. 
\end{definition}
\noindent The computational complexity for the training phase in the pattern layer is $\mathcal{O}\left(N \left(d + K\right)\right)$, where $N$ is the number of training samples, $d$ is the dimension of the input feature space and $K$ accounts for the computational cost of $\Phi\left(\cdot\right)$. In the summation layer, the computational complexity is $\mathcal{O}\left(C N\right)$, where $C$ is the total number of classes. Therefore, the overall computational complexity for SkewPNN is $\mathcal{O}\left(N \left(d + K\right)\right)$ with the dominant term being the kernel computation, whereas, for PNN with Gaussian kernel, the complexity is $\mathcal{O}\left(N d\right)$ for the training phase. For PNN, the total complexity for the distance calculation step is $\mathcal{O}\left(N d\right)$, the complexity of the Gaussian kernel function is constant, and the summation layer remains $\mathcal{O}\left(C N\right)$ operation for $C$ classes. Therefore, the overall complexity for classifying a single input vector is $\mathcal{O}\left(N d\right)$ as $Nd$ dominates for large $d$ or $N$. Adding the skewness parameter $\alpha$ increases the complexity of model selection. However, the inference complexity for the traditional neural network with two hidden layers is $\mathcal{O}\left(N\left(d + N\right)\right)$, much higher than SkewPNN. This is because the training is instantaneous in PNN and SkewPNN compared to the traditional neural network, which uses iterative gradient descent. For the prediction phase having $M$ test samples, the complexity for PNN is $\mathcal{O}\left(MNd\right)$, whereas for SkewPNN, it is $\mathcal{O}\left(MN(d+K)\right)$. BA-SkewPNN has higher training complexity as compared to SkewPNN due to hyperparameter tuning, which is $\mathcal{O}\left(TBPN(d+K)\right)$, where $T$ is the number of iterations for $B$ Bats for $P$ number of cross-validation. However, the computational complexity of SkewPNN and BA-SkewPNN is much lower than that of modern deep learning methods.}

\subsection{Consistency}
\noindent The Parzen window estimator~\cite{parzen1962estimation} for the PDF is expressed in Eqn. (\ref{Eq2}). $K$ represents the kernel function, which, in our case, is the skew-normal kernel function. Here, $h(n)$ approaches zero as the data points increase. It is chosen in a way that decreases the width of the window, allowing for more localized estimation as the dataset size grows. The specific choice of $h(n)$ can depend on the problem and the characteristics of the data. 
The standard way to define consistency is that the expected error gets smaller as the estimates are based on a larger dataset. This is of particular interest since it will ensure that the true distribution will be approached in a smooth manner~\cite{parzen1962estimation}. Conditions under which this happens for the density estimates are given by the following theorem:

\begin{theorem}
    Suppose $K(x; \xi, \sigma^2, \alpha)$ (as in Eqn. \ref{Eq_10}) is a Borel function satisfying the conditions:
    \begin{enumerate}
        \item[(A1)] $\underset{- \infty < x < \infty}{\operatorname{sup}} \left|K(x; \xi, \sigma^2, \alpha) \right| < \infty$,

        \item[(A2)] $\int_{- \infty}^{\infty} \left|K(x; \xi, \sigma^2, \alpha) \right| dx < \infty$,

        \item[(A3)] $\underset{x \rightarrow \infty}{\lim} \left|x K(x; \xi, \sigma^2, \alpha) \right| = 0$,

        \item[(A4)] $\int_{- \infty}^{\infty} K(x; \xi, \sigma^2, \alpha) dx = 1$,
    \end{enumerate}
along with the conditions in Eqn. (\ref{Eq2}) and Eqn. (\ref{Eq_9}), then the estimate $f_n(x)$ is consistent in quadratic mean in the sense that
$$
    \mathbb{E}\left|f_nx) - g(x) \right| \rightarrow 0 \text{ as } n \rightarrow \infty. 
$$
\end{theorem}

\begin{proof}
To show that $K(x; \xi, \sigma^2, \alpha)$ satisfying (A1) with skew-normal kernel density, we need to find the mode of the skew-normal distribution. For that, we establish the following result on log-concavity; that is, the logarithm of its density is a concave function \cite{azzalini2013skew}.

\begin{prop}
    The distribution $K(x; \xi, \sigma^2, \alpha)$ in SkewPNN is log-concave.
\end{prop}
\begin{proof}
    It suffices to prove this for the case when $\xi = 0$ and $\sigma^2 = 1$ since a change of location and scale do not alter the property. To prove that $\log K(x; \xi, \sigma^2, \alpha)$ is a concave function of $x$, it is sufficient to show that the second derivative of $\log K(x; \xi, \sigma^2, \alpha)$ is negative for all $x$, following~\cite{azzalini2013skew}:
    \begin{equation}\label{Eq_6}
    \frac{d^2}{dx^2} \log K(x; \xi, \sigma^2, \alpha) = -1 + \frac{-\alpha^2 \phi \left(\alpha x \right)}{\Phi \left(\alpha x \right)} \left[ \frac{\phi\left(\alpha x \right)}{\Phi \left(\alpha x \right)} + \alpha x\right].        
    \end{equation}
To show that the R.H.S. of Eqn. (\ref{Eq_6}) is negative, it is sufficient to show that $B\left( \alpha x\right) = \left\{ \frac{\phi\left(\alpha x \right)}{\Phi \left(\alpha x \right)} + \alpha x\right\}$ is positive for all $\alpha x$ since $\phi(\alpha x)$ and $\Phi(\alpha x)$ are positive for all $x$.

\vspace{0.1cm}
\noindent \underline{Case I:} If $\alpha x \geq 0$, then $B(\alpha x)$ is clearly positive.

\vspace{0.1cm}
\noindent \underline{Case II:} If $\alpha x < 0$, let $v = -\alpha x$. 

\noindent Then $\phi \left(\alpha x\right) = \phi \left(-\alpha x\right) = \phi \left(v\right)$ and $\Phi \left(\alpha x\right) = 1 - \Phi \left(-\alpha x\right) = 1 - \Phi \left(v\right)$. Therefore, we get
$$ B\left( \alpha x \right) = \frac{\phi(v)}{1 - \Phi(v)} - v = r(v) - v,$$
where $r(v)$ is the failure rate of a standard normal random variable. Since it is known (see~\cite{azzalini1985class}) that $r(v) > v$ for all $v$, the assertion is proved.
\end{proof}
\noindent Hence, it is evident that the mode is unique. An interesting observation for the $K(x; \xi, \sigma^2, \alpha)$ is that the unimodality holds (univariate case), which coincides with the log-concavity of the distribution. We denote by $\xi + \sigma m_0 (\alpha)$ the mode of the skew-normal distribution. For any general $\alpha$, no explicit expression of $m_0(\alpha)$ is available and can be done using any iterative method such as Newton-Raphson numerical optimization method. However, the authors in~\cite{azzalini2013skew} obtained a simple and practically accurate approximation as given below:
$$m_0(\alpha) \approx \mu_z - \frac{\gamma_1 \sigma_z}{2} - \frac{\operatorname{sgn}(\alpha)}{2} \operatorname{exp}\left( - \frac{2 \pi}{|\alpha|}\right),$$
where $\mu_z = \sqrt{\frac{2}{\pi}} \delta$ and $\sigma_z = \sqrt{1 - \mu_z^2}$. In addition, the authors in~\cite{azzalini2013skew} obtained via numerical simulation that the mode occurs at $\alpha \approx 1.548$, $\delta = 0.8399$, where its value is 0.5427. Therefore, the maximum value of the $K(x; \xi, \sigma^2, \alpha)$ is finite.\\

\noindent (A2) and (A4) are trivially satisfied as $K(x; \xi, \sigma^2, \alpha)$ is a PDF of the skew-normal distribution as defined in Eqn. (\ref{Eq_10}).\\

\noindent (A3) To show the limiting condition, we need the existence of the expectation.
$$
    \mathbb{E}_{K(x; \xi, \sigma^2, \alpha)}[X] = \underset{u \rightarrow \infty}{\lim} \int_{- \infty}^u x K(x; \xi, \sigma^2, \alpha) dx = \int_{- \infty}^{\infty} x K(x; \xi, \sigma^2, \alpha) dx < \infty. 
$$
Since ~\cite{azzalini1985class} derived the expression for the mean of skew-normal density function as follows: If $Y \sim SN\left(\xi, \sigma^2, \alpha \right)$, then $\mathbb{E}[Y] = \xi + \sigma \sqrt{\frac{2}{\pi}} \delta,$ where $\delta = \frac{\alpha}{\sqrt{1 + \alpha^2}}$ for the univariate case and it is well defined. Now, for $u \geq 0$ 
$$
    \int_u^{\infty} x K(x; \xi, \sigma^2, \alpha) dx \geq u \int_u^{\infty} K(x; \xi, \sigma^2, \alpha) dx = u\left[1 - K_v(u; \xi, \sigma^2, \alpha)\right],
$$
where $K_v(u; \xi, \sigma^2, \alpha)$ is the CDF of skew-normal kernel used in SkewPNN. Therefore, it follows that
$$
    \underset{u \rightarrow \infty}{\lim} \left[ \mathbb{E}_K \left[X\right] - \int_{-\infty}^{\infty} x K(x; \xi, \sigma^2, \alpha) dx \right] = \underset{u \rightarrow \infty}{\lim} \int_{u}^{\infty}x K(x; \xi, \sigma^2, \alpha) dx = 0
$$
as in the limit, the term $\int_{-\infty}^u x K(x; \xi, \sigma^2, \alpha)dx$ approaches to the expectation. By the inequality and the non-negativity of the integrand, we have the main result. Parzen \cite{parzen1962estimation} proved that the estimate $f_n(x)$ is consistent in quadratic mean if conditions (A1)-(A4) are met, which satisfies in this case.
\end{proof}

\begin{remark}
    {\color{black}Similar to PNN, SkewPNN also approaches to Bayes optimal classification as the training set size increases.} 
\end{remark}

\section{Experimental Analysis}\label{sec4}
{\color{black} In this section, we examined the effectiveness of our proposals on both synthetic and real-world datasets. First, we validate our proposed method on toy datasets to show the impact of the proposed mechanisms on finding smooth decision boundaries. Furthermore, numerical experiments were also conducted on real-world datasets collected from various applied fields.

\subsection{Experimental Setting}
In real-world data experiments, the proposed algorithmic-level solutions, namely SkewPNN and BA-SkewPNN, were compared with 16 competitive classification methods, of which 3 are data-level and 13 are algorithm-level approaches. A detailed description of these competing methods is presented in Table \ref{table2}. For these existing methods, we followed the standard implementations with the default parameter settings as described in the references mentioned in Table \ref{table2}. The performance of the traditional classifiers is usually measured based on classification accuracy. Although it is beneficial in balanced classification, it is inappropriate for imbalanced classification scenarios due to its preference for majority classes \cite{tholke2023class, aguilar2024classification, das2022supervised}. A recent study \cite{mcdermott2024closer} showed that the area under the precision-recall curve (AUPRC) is not an appropriate performance measure in cases of class imbalance. Theoretical and empirical findings of \cite{mcdermott2024closer} showed that AUPRC could be a harmful metric since it favors model improvements in sub-populations with more frequent majority class labels, heightening algorithmic disparities. Based on these insights, we carefully selected two evaluation indices for evaluating the experimental results, following \cite{chakraborty2020hellinger, chakraborty2020superensemble, ding2023rvgan}. Hence, we adopted two popularly used performance metrics for comparing state-of-the-art methods with our proposed models that can balance the performance between the majority and minority classes: F1 score (F1) and area under the receiver operating characteristic curve (AUC-ROC), following \cite{chakraborty2020hellinger, chakraborty2020superensemble, ding2023rvgan, akash2019inter, wei2023minority}. F1 is one of the most popular criteria for evaluating the classification accuracy of imbalanced data, which is the harmonic average of precision and recall \cite{sajjadi2018assessing}. F1 ranges between 0 and 1. The higher the F1, the better the classification effect for imbalanced data scenarios.} This metric is particularly valuable when achieving a trade-off between accurate identification and minimizing false positives is essential. {\color{black} Another index used to measure the classification effect of imbalanced data is the area under the curve (AUC) of the receiver operating characteristic (ROC).} The AUC-ROC serves as a critical metric for assessing the effectiveness of classifiers in distinguishing between positive (majority) and negative (minority) classes~\cite{hanley1982meaning}. It visually illustrates the classifier's performance by plotting $sensitivity$ against $1 - specificity$ at various thresholds, offering a consistent measure of the classifier's ability to prioritize positive instances over negatives, regardless of the threshold used. Here, sensitivity, also referred to as recall, quantifies the proportion of actual positives correctly predicted by the classifier. In contrast, specificity measures the proportion of actual negatives correctly identified by the classifier. Precision denotes the ratio of true positive predictions to all predicted positives, while sensitivity measures the proportion of true positive predictions among actual positives. Hence, the higher the value of AUC-ROC, the better the overall classification effect.


\subsection{Experiments on Simulated datasets}
{\color{black}This section provides a comprehensive analysis of the classification performance of the proposed SkewPNN model and two baseline frameworks using synthetic datasets generated with the \texttt{scikit-learn} and \texttt{imbalanced-learn} libraries in Python \cite{lemaitre2017imbalanced}.} This analysis aims to demonstrate the distinct decision boundaries produced by these classifiers. {\color{black}We use three synthetic datasets to evaluate the performance of classification algorithms under varying levels of class imbalance. The half-moon dataset is generated using the make$\_$moons function from the \texttt{scikit-learn} library, representing two interleaving half-circles. The concentric circles dataset is created using the make$\_$circles function, comprising observations falling in concentric circles.} Finally, an intertwined spiral dataset is simulated, representing two-dimensional points arising from a complex interaction of two spirals. {\color{black}In total, 714 observations are generated across these datasets, with Gaussian noise (standard deviation = 0.35) added to increase variability. To assess the impact of varying class imbalance, three setups are used where the imbalanced ratio (IR), the proportion of majority over minority samples, are set to 4.0, 9.0, and 19.0.} In all experiments, 80\% of the data samples are allocated for training, while the remaining 20\% are reserved for testing the classifiers' performance. For evaluation, we compare the decision boundaries of the proposed SkewPNN method against those of the HDDT and PNN classifiers. The hyperparameter values are set as $\sigma = 0.2$ and $\alpha = -2$ for the SkewPNN model, for the HDDT the maximum depth of the tree is 5, and for PNN $\sigma = 0.1$. {\color{black}Figs. \ref{fig:half_moon}, \ref{fig:full_moon}, and \ref{fig:interwined} showcase the decision boundaries of the classification algorithms for the half-moon, concentric circles, and intertwined spiral datasets, respectively. As depicted in the plots, the SkewPNN model can correctly classify most of the observations and achieves the highest AUC-ROC (reported in the lower-right corner of each subplot) compared to the baseline approaches. These results affirm that the SkewPNN framework effectively handles complex data structures with nonlinear manifolds and accurately classifies both majority and minority samples across datasets with varying class imbalances. Moreover, the plots show that the HDDT-generated decision boundaries are axis-parallel lines (non-smooth), whereas SkewPNN-generated decision boundaries are very smooth.} This outcome demonstrates that the proposed approach is well-suited for addressing highly imbalanced data structures.



\begin{figure*}[!ht]
\centering
\includegraphics[width=1\textwidth]{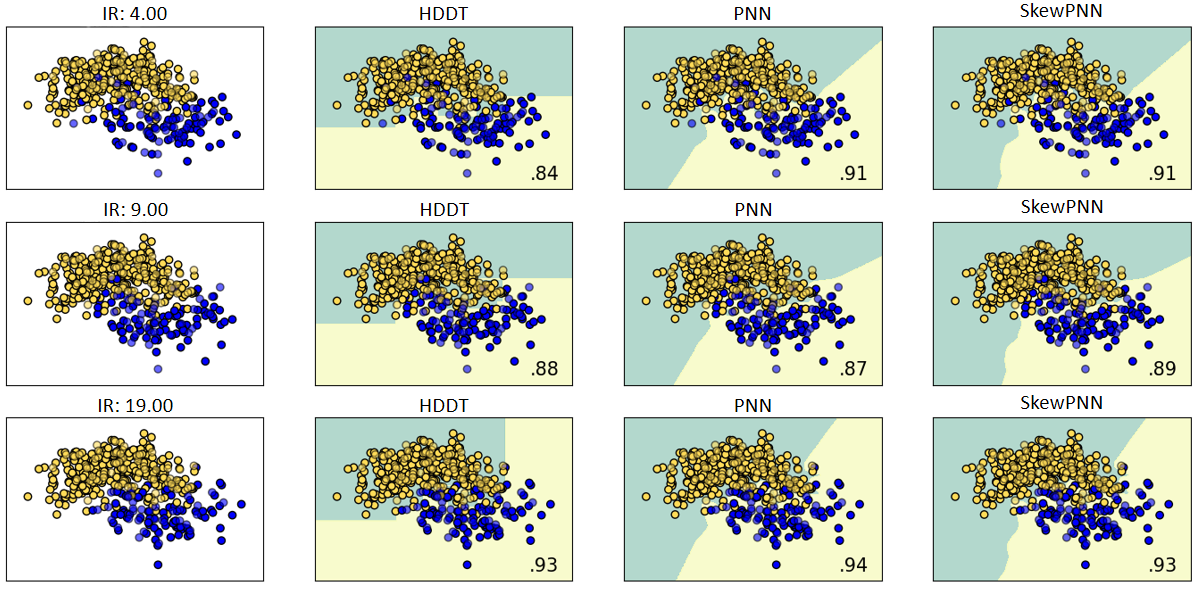}
\caption{{\color{black}Comparison of HDDT, PNN, and SkewPNN classifiers on the synthetically generated half-moon dataset. Training points are displayed in solid colors, while testing points are shown with semi-transparency. The decision boundaries generated by the respective classifiers separating the two classes are depicted by green-shaded regions in subplots. The AUC-ROC score for the test set is presented in the lower-right corner of each plot.}}
\label{fig:half_moon}
\end{figure*}

\begin{figure*}[!ht]
\centering
\includegraphics[width=1\textwidth]{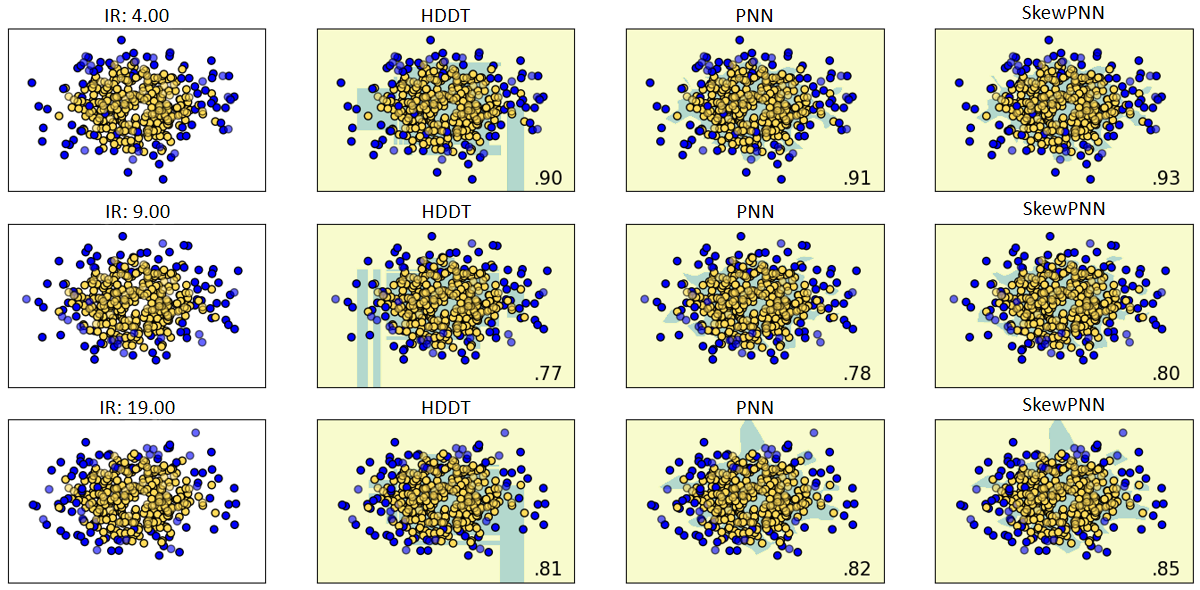}
\caption{{\color{black}Comparison of HDDT, PNN, and SkewPNN classifiers on the synthetically generated concentric circles dataset. Training points are displayed in solid colors, while testing points are shown with semi-transparency. The decision boundaries generated by the respective classifiers separating the two classes are depicted by green-shaded regions in subplots. The AUC-ROC score for the test set is presented in the lower-right corner of each plot.}}
\label{fig:full_moon}
\end{figure*}

\begin{figure*}[!ht]
\centering
\includegraphics[width=1\textwidth]{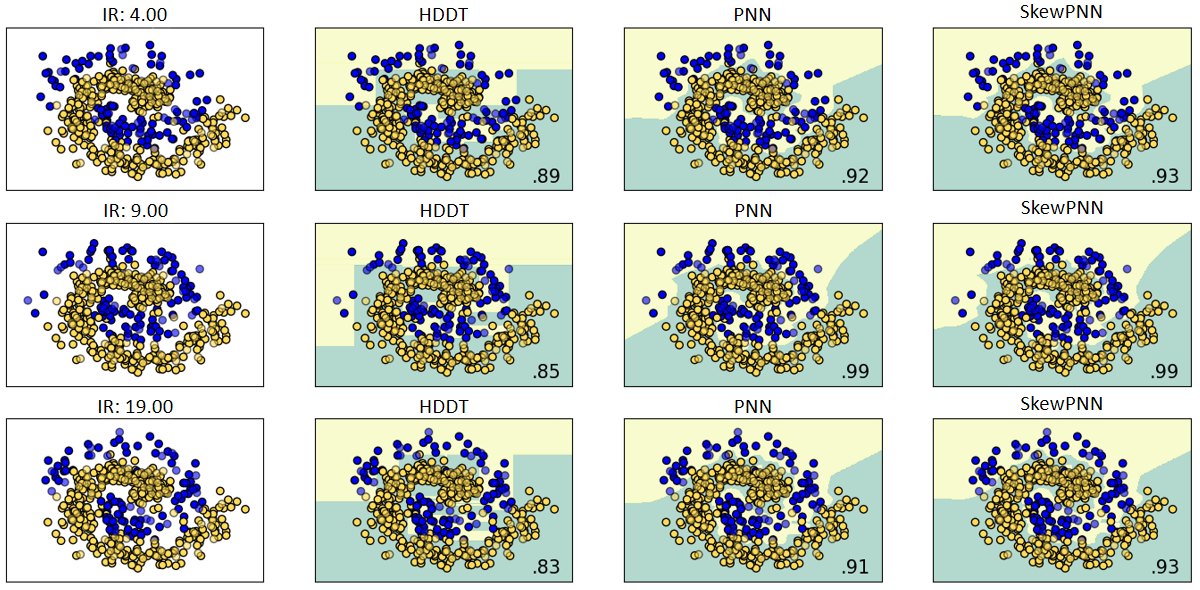}
\caption{{\color{black}Comparison of HDDT, PNN, and SkewPNN classifiers on the synthetically generated intertwined spiral dataset. Training points are displayed in solid colors, while testing points are shown with semi-transparency. The decision boundaries generated by the respective classifiers separating the two classes are depicted by green-shaded regions in subplots. The AUC-ROC score for the test set is presented in the lower-right corner of each plot.}}
\label{fig:interwined}
\end{figure*}


\subsection{Experiments on Real-world datasets}
{\color{black}This section assesses the performance of the proposed algorithms for classifying real-world datasets with balanced and imbalanced class distribution. The following subsections describe the datasets used in our analysis, summarizes the performance of the proposed approaches with state-of-the-art classification methods, and evaluates the statistical significance of the performance improvements.
\subsubsection{Datasets}\label{Sec_dataset}
To show the effectiveness of our proposed methods, we employed 35 datasets chosen from various fields like biology, medicine, business, and finance. Table \ref{table1} groups the datasets into two categories: Imbalanced datasets and Balanced datasets. These datasets are collected from two popular public sources such as the UCI Machine Learning Repository \cite{uciml2017dua} and KEEL Imbalanced Datasets \cite{alcal2011keel}. The imbalanced ratio (IR) between the samples of majority and minority classes is displayed in Table \ref{table1}. By definition, a higher value of IR indicates the dataset is highly imbalanced. Following \cite{ding2023rvgan, akash2019inter, wei2023minority, aler2020study}, we selected 20 imbalanced datasets from KEEL where the IR ranges between 2.49 and 72.69. We also selected the balanced datasets, following \cite{zhu2023classification, akash2019inter} from UCI, where the IR is close to 1. These 35 datasets also vary in terms of the number of features from 3 to 100 and the number of samples from 106 to 5472, following \cite{ding2023rvgan}. Out of the 35 datasets, we consider almost 25\% of the datasets as multiclass data to show the strength of our proposed methods in multiclass scenarios. The number of features ranges between 5 and 90, while the number of samples lies between 132 and 1728 for the multiclass datasets considered in this study.}

\begin{table}[!ht]
\caption{List of Imbalanced datasets and Balanced datasets utilized in experimentation. }
\label{table1}
\centering
\resizebox{0.98\textwidth}{!}{\begin{tabular}{|l|l|c|c|c|l|llcccl}
\hline
Sr. No. & Imbalanced & No. of & No. of & No.of & IR    & \multicolumn{1}{l|}{Sr. No.} & \multicolumn{1}{l|}{Balanced}    & \multicolumn{1}{c|}{No.of} & \multicolumn{1}{c|}{No.of} & \multicolumn{1}{c|}{No.of} & \multicolumn{1}{l|}{IR}   \\ 
 & Datasets & Samples & Attributes & Classes &    & \multicolumn{1}{l|}{ } & \multicolumn{1}{l|}{Datasets}    & \multicolumn{1}{c|}{Samples} & \multicolumn{1}{c|}{Attributes} & \multicolumn{1}{c|}{Classes} & \multicolumn{1}{l|}{ }   \\ \hline
ID1     & ecoli-0-1-4-7 vs 2-3-5-6               & 336       & 7            & 2         & 15.80 & \multicolumn{1}{l|}{BD1}     & \multicolumn{1}{l|}{heart-c}     & \multicolumn{1}{c|}{1025}      & \multicolumn{1}{c|}{13}           & \multicolumn{1}{c|}{2}         & \multicolumn{1}{l|}{1.05} \\ \hline
ID2     & haberman             & 306       & 3            & 2         & 2.77  & \multicolumn{1}{l|}{BD2}     & \multicolumn{1}{l|}{pima}        & \multicolumn{1}{c|}{768}       & \multicolumn{1}{c|}{8}            & \multicolumn{1}{c|}{2}         & \multicolumn{1}{l|}{1.86} \\ \hline
ID3     & vehicle3             & 846       & 18           & 2         & 2.99  & \multicolumn{1}{l|}{BD3}     & \multicolumn{1}{l|}{tic-tac-toe} & \multicolumn{1}{c|}{958}       & \multicolumn{1}{c|}{9}            & \multicolumn{1}{c|}{2}         & \multicolumn{1}{l|}{1.88} \\ \hline
ID4     & yeast-0-3-5-9 vs 7-8 & 506       & 8            & 2         & 9.12  & \multicolumn{1}{l|}{BD4}     & \multicolumn{1}{l|}{australian}  & \multicolumn{1}{c|}{690}       & \multicolumn{1}{c|}{14}           & \multicolumn{1}{c|}{2}         & \multicolumn{1}{l|}{1.24} \\ \hline
ID5     & yeast-2 vs 4         & 514       & 8            & 2         & 9.07  & \multicolumn{1}{l|}{BD5}     & \multicolumn{1}{l|}{bupa}        & \multicolumn{1}{c|}{345}       & \multicolumn{1}{c|}{6}            & \multicolumn{1}{c|}{2}         & \multicolumn{1}{l|}{1.37} \\ \hline
ID6     & page-block                 & 5472      & 10           & 2         & 8.78  & \multicolumn{1}{l|}{BD6}     & \multicolumn{1}{l|}{crx}         & \multicolumn{1}{c|}{690}       & \multicolumn{1}{c|}{15}           & \multicolumn{1}{c|}{2}         & \multicolumn{1}{l|}{1.24} \\ \hline
ID7     & abalone-20 vs 8-9-10 & 1916      & 8            & 2         & 72.69 & \multicolumn{1}{l|}{BD7}     & \multicolumn{1}{l|}{ion}         & \multicolumn{1}{c|}{351}       & \multicolumn{1}{c|}{33}           & \multicolumn{1}{c|}{2}         & \multicolumn{1}{l|}{1.78} \\ \hline
ID8     & abalone9-18          & 731       & 8            & 2         & 16.40 & \multicolumn{1}{l|}{BD8}     & \multicolumn{1}{l|}{hill valley} & \multicolumn{1}{c|}{606}       & \multicolumn{1}{c|}{100}          & \multicolumn{1}{c|}{2}         & \multicolumn{1}{l|}{1.02} \\ \hline
ID9     & yeast4               & 1484      & 8            & 2         & 28.09 & \multicolumn{1}{l|}{BD9}     & \multicolumn{1}{l|}{monks-2}     & \multicolumn{1}{c|}{432}       & \multicolumn{1}{c|}{6}            & \multicolumn{1}{c|}{2}         & \multicolumn{1}{l|}{1.11} \\ \hline
ID10    & Ecoli-0-2-3-4 vs 5   & 202       & 7            & 2         & 9.10   & \multicolumn{1}{l|}{BD10}    & \multicolumn{1}{l|}{spect-f}     & \multicolumn{1}{c|}{80}        & \multicolumn{1}{c|}{44}           & \multicolumn{1}{c|}{2}         & \multicolumn{1}{l|}{1.00}    \\ \hline
ID11    & glass-0-1-4-6 vs 2   & 205       & 9            & 2         & 11.05 & \multicolumn{1}{l|}{BD11}    & \multicolumn{1}{l|}{wisconsin}   & \multicolumn{1}{c|}{683}       & \multicolumn{1}{c|}{9}            & \multicolumn{1}{c|}{2}         & \multicolumn{1}{l|}{1.85} \\ \hline
ID12    & appendicitis         & 106       & 7            & 2         & 4.07  & \multicolumn{1}{l|}{{\color{black}BD12}}    & \multicolumn{1}{l|}{{\color{black}tae}}   & \multicolumn{1}{c|}{{\color{black}151}}       & \multicolumn{1}{c|}{{\color{black}5}}            & \multicolumn{1}{c|}{{\color{black}3}}         & \multicolumn{1}{l|}{{\color{black}1.06}} \\ \hline
ID13    & car-vgood            & 1728      & 6            & 2         & 25.58 & \multicolumn{1}{l|}{{\color{black}BD13}}    & \multicolumn{1}{l|}{{\color{black}wine}}   & \multicolumn{1}{c|}{{\color{black}178}}       & \multicolumn{1}{c|}{{\color{black}13}}            & \multicolumn{1}{c|}{{\color{black}3}}         & \multicolumn{1}{l|}{{\color{black}1.47}} \\ \hline
ID14    & transfusion     & 748       & 4           &   2  & 3.20  & \multicolumn{1}{l|}{{\color{black}BD14}}    & \multicolumn{1}{l|}{{\color{black}hayes-roth}}   & \multicolumn{1}{c|}{{\color{black}132}}       & \multicolumn{1}{c|}{{\color{black}5}}            & \multicolumn{1}{c|}{{\color{black}3}}         & \multicolumn{1}{l|}{{\color{black}1.70}} \\ \hline
ID15    & lipid-indian-liver             &    583   & 10        &   2       &  2.49 & \multicolumn{1}{l|}{{\color{black}BD15}}    & \multicolumn{1}{l|}{{\color{black}movement\_libras}}   & \multicolumn{1}{c|}{{\color{black}360}}       & \multicolumn{1}{c|}{{\color{black}90}}            & \multicolumn{1}{c|}{{\color{black}15}}         & \multicolumn{1}{l|}{{\color{black}1.00}} \\ \hline
ID16    & car-good             & 1728      & 6            & {\color{black}4}         & 18.61 &         &  & \multicolumn{1}{l}{}           & \multicolumn{1}{l}{}              & \multicolumn{1}{l}{}           &                           \\ \cline{1-6}
ID17    & dermatology-6 & 358       & 34            &   {\color{black}6}       & 5.55  &                              &                                  & \multicolumn{1}{l}{}           & \multicolumn{1}{l}{}              & \multicolumn{1}{l}{}           &                           \\ \cline{1-6}
ID18    & flare-F   &   1066     & 11            &   {\color{black}6}       & 7.69  &                              &                                  & \multicolumn{1}{l}{}           & \multicolumn{1}{l}{}              & \multicolumn{1}{l}{}           &                           \\ \cline{1-6}

{\color{black}ID19}    & {\color{black}new-thyroid}   & {\color{black}215}       & {\color{black}5}            & {\color{black}3}         & {\color{black}5.00}  &                              &                                  & \multicolumn{1}{l}{}           & \multicolumn{1}{l}{}              & \multicolumn{1}{l}{}           &                           \\ \cline{1-6}
{\color{black}ID20}    & {\color{black}DataUserMod}   & {\color{black}258}       & {\color{black}5}            & {\color{black}4}         & {\color{black}3.67}  &                              &                                  & \multicolumn{1}{l}{}           & \multicolumn{1}{l}{}              & \multicolumn{1}{l}{}           &                           \\ \cline{1-6}
\end{tabular}}
\end{table}

\subsubsection{Analysis of Real-world Datasets }\label{{Sec_real_world_result}}
To assess the efficacy of the proposed SkewPNN and BA-SkewPNN, we implemented them on datasets characterized by balanced and imbalanced data distributions, as detailed in Table~\ref{table1}. Through a series of meticulously designed experiments on these standardized datasets, we systematically compared the performance of our proposed methods against a comprehensive array of techniques for class-imbalanced (also balanced) learning. Our methodology commenced with the random shuffling of observations within each dataset, followed by the application of z-score normalization to ensure uniformity. Subsequently, we conducted 10-fold cross-validation, employing distinct and randomly generated training and test sets for each iteration. Moreover, our assessment encompassed a comprehensive comparative analysis between our proposed classifiers and a range of benchmark methods. This empirical evaluation allowed us to assess an in-depth analysis of performance using key metrics such as AUC-ROC and F1-score.


\begin{table}[]
    \centering
    \tiny
\caption{List of classifiers utilized in experimental evaluation.}
\resizebox{1.0\textwidth}{!}{\begin{tabular}{|c|c|c|c|c|}
\hline Approaches & Methods & Abbreviations & Key Aspects & Publication Year \\
\hline
\multirow{6}{*}{Data-Level} 
& \multirow{2}{*}{Synthetic Minority Oversampling TEchnique + CART~\cite{chawla2002smote}}  & \multirow{2}{*}{SC} & \multirow{2}{*}{Oversampling Method} & 
\multirow{2}{*}{2002} \\ & & & & \\
\cline{2-5}
& {\color{black}Synthetic Minority Oversampling TEchnique  \cite{sauglam2022novel}} & \multirow{2}{*}{{\color{black}SWBC}} & 
\multirow{2}{*}{{\color{black}Oversampling with Boosting Approach}} & \multirow{2}{*}{{\color{black}2022}} \\ & {\color{black}with boosting + CART} & & & \\
\cline{2-5}
& \multirow{2}{*}{{\color{black}Generalized tHreshOld ShifTing + RF \cite{esposito2021ghost}}} & \multirow{2}{*}{{\color{black}GHOST}} & \multirow{2}{*}{{\color{black}Thresholding Method}} & \multirow{2}{*}{{\color{black}2021}} \\ & & & & \\
\hline \multirow{17}{*}{Algorithm-Level} & Decision Tree~\cite{breiman2017classification} & DT &  
\multirow{5}{*}{Tree-based Classifiers} & 1984 \\
 & Neural Network~\cite{ripley1994neural} & NN & & 1994 \\ 
 & Surface-to-Volume Ratio Tree~\cite{zhu2023classification}& SVRT & & {\color{black}2023} \\ 
 & Hellinger Distance Decision Tree~\cite{cieslak2012hellinger} & HDDT & & 2012 \\ 
 & {\color{black}Inter-node HDDT~\cite{akash2019inter}} & {\color{black}iHDDT} & & {\color{black}2019} \\ 
 \cline{2-5} 
 & Random Forest~\cite{breiman2001random} & RF & \multirow{6}{*}{Ensemble Learning Methods}  &  2001 \\
 & Ada Boost~\cite{freund1995desicion} & AB & & 1995 \\ 
 & Extreme Gradient Boosting~\cite{chen2015xgboost} & XGB & & 2015 \\
 & Imbalanced XG Boost~\cite{wang2020imbalance} & IXGB& & {\color{black}2020} \\
 & Hellinger Distance Random Forest~\cite{su2015improving} & HDRF & & 2015 \\
 \cline{2-5}
& Hellinger Net~\cite{chakraborty2020hellinger} & HNet & \multirow{5}{*}{Neural Network Classifiers}  &  {\color{black}2020} \\
& Classical PNN~\cite{specht1990probabilistic} & PNN & &  1990 \\
& Bat Algorithm based PNN~\cite{yang2019ba} & BA-PNN & & {\color{black}2019} \\
& \textit{Skew-Normal PNN} & SkewPNN & & \textit{Proposed} \\ 
& \textit{Bat Algorithm based Skew-Normal PNN} & BA-SkewPNN & & \textit{Proposed} \\ \hline
\end{tabular}}
\label{table2}
\end{table}
To implement the conventional PNN where the Gaussian kernel is employed, the smoothing parameter ($\sigma$) was systematically varied across the range of $0.01$ or $0.1$ till approximate constant (either $1$ or $10$), with increments of $0.05$ or $0.1$ for all samples. Subsequently, we explored an alternative approach. In this variant, we integrated the Bat algorithm to dynamically assign optimal and distinct smoothing parameters for each sample during the PDF calculations. This model, denoted as BA-PNN, operated with two distinct smoothing parameter intervals: $[0.01, 1.0)$ and $[0.1, 1.0)$. Moving forward, we conducted experiments utilizing our proposed SkewPNN and BA-SkewPNN. In the SkewPNN, we leveraged the skew-normal kernel to estimate the values within pattern layers. We choose the skewness parameter for SkewPNN from a pre-specified range of [-6, 6]. Similarly, the smoothing parameter is chosen across samples, ranging from $0.01$ or $0.1$ to approximate constant (either $1$ or $10$), with intervals of $0.05$ or $0.1$. In the case of the BA-SkewPNN model, the BA algorithm was employed to assign the optimal and diverse smoothing parameters for the PDF calculations across all the samples. Likewise, two distinct intervals were maintained for selecting smoothing parameters: $[0.01, 1.0)$ or $[0.1, 1.0)$. Notably, the fitness function remained consistent for both BA-PNN and BA-SkewPNN.
\begin{table*}[!ht]
\caption{Experimental results encompassing AUC-ROC and F1-score (mean and standard deviation in bracket) analyses conducted on imbalanced datasets listed in Table~\ref{table1}. The highest value of the metrics for a particular dataset is highlighted in {\underline{\textbf{bold}}}.}
\label{table3}
\resizebox{1.0\textwidth}{!}{
\begin{tabular}{|l|l|c|c|c|c|c|c|c|c|c|c|c|c|c|c|c|c|c|c|}
\hline
\multirow{2}{*}{Data} & \multirow{2}{*}{Measure}  & \multirow{2}{*}{DT}  & \multirow{2}{*}{NN}    & \multirow{2}{*}{AB}             & \multirow{2}{*}{RF}    & \multirow{2}{*}{SC}    & \multirow{2}{*}{SVRT}   & \multirow{2}{*}{HDDT}           & \multirow{2}{*}{HDRF}           & \multirow{2}{*}{XGB}            & \multirow{2}{*}{IXGB}           & \multirow{2}{*}{HNet}         & \multirow{2}{*}{{\color{black}iHDDT}} & \multirow{2}{*}{{\color{black}GHOST}} & \multirow{2}{*}{{\color{black}SWBC}}   & \multirow{2}{*}{PNN}            & \multirow{2}{*}{BA-PNN}         & \multirow{2}{*}{SkewPNN}        & \multirow{2}{*}{BA-SkewPNN}     \\  %
& &  &  & & & & & & & & & & & & & & & &      \\ \hline
\multirow{4}{*}{ID1}          & \multirow{2}{*}{AUC-ROC}  & 0.819          & 0.859 & 0.818          & 0.820 & 0.765 & 0.862 & 0.787          & 0.803          & 0.943          & 0.968          & 0.900       & {\color{black}0.509}  & {\color{black}0.900} & {\color{black}0.861} & 0.908          & 0.921          & 0.933          & \underline{\textbf{0.980}} \\ 
                            &  &  {\color{black}(0.141)} & {\color{black}(0.125)}  & {\color{black}(0.108)} & {\color{black}(0.134)} & {\color{black}(0.102)} & {\color{black}(0.066)} & {\color{black}(0.123)} & {\color{black}(0.145)} & {\color{black}(0.055)} & {\color{black}(0.030)} & {\color{black}(0.056)} & {\color{black}(0.041)} & {\color{black}(0.075)} & {\color{black}(0.119)} & {\color{black}(0.118)} & {\color{black}(0.012)} & {\color{black}(0.132)} & {\color{black}(0.047)} \\ 
                              & \multirow{2}{*}{F1} & 0.622          & 0.756 & 0.716          & 0.726 & 0.612 & 0.737 & 0.601          & 0.700          & 0.526          & \underline{\textbf{0.893}} & 0.889       & {\color{black}0.000} &  {\color{black}0.689} & {\color{black}0.753} &  0.744          & 0.857          & 0.774          & 0.881          \\
\hline

\end{tabular}}
\end{table*}

\begin{table*}[!ht]
\caption{Experimental results encompassing AUC-ROC and F1-score (mean and standard deviation in bracket) analyses conducted on balanced datasets listed in Table~\ref{table1}. The highest value of the metrics for a particular dataset is highlighted in {\underline{\textbf{bold}}}.}
\label{table4}
\resizebox{1.0\textwidth}{!}{\begin{tabular}{|l|l|c|c|c|c|c|c|c|c|c|c|c|c|c|c|c|c|c|c|}
\hline
\multirow{2}{*}{Data} & \multirow{2}{*}{Measure}  & \multirow{2}{*}{DT}    & \multirow{2}{*}{NN}      & \multirow{2}{*}{AB}  & \multirow{2}{*}{RF}    & \multirow{2}{*}{SC}  & \multirow{2}{*}{SVRT}   & \multirow{2}{*}{HDDT}  & \multirow{2}{*}{HDRF} & \multirow{2}{*}{XGB} & \multirow{2}{*}{IXGB}         & \multirow{2}{*}{HNet} & \multirow{2}{*}{{\color{black}iHDDT}} & \multirow{2}{*}{{\color{black}GHOST}}  & \multirow{2}{*}{{\color{black}SWBC}} &\multirow{2}{*}{PNN} & \multirow{2}{*}{BA-PNN} & \multirow{2}{*}{SkewPNN} & \multirow{2}{*}{BA-SkewPNN} \\
& & & & & & & & & & & & & & & & &  &  &      \\\hline 
\multirow{4}{*}{BD1}  & \multirow{2}{*}{AUC-ROC}  & 0.725 & 0.827          & 0.791          & 0.795 & 0.996 & 0.524 & 0.733 & 0.810  & 0.814 & 0.795  & 0.745 & {\color{black}0.634} & {\color{black}0.956} & {\color{black}0.996} &   \underline{\textbf{1.000}} &  \underline{\textbf{1.000}} & \underline{\textbf{1.000}} & \underline{\textbf{1.000}} \\ 
                    & & {\color{black}(0.073)} & {\color{black}(0.041)} & {\color{black}(0.058)} & {\color{black}(0.043)} & {\color{black}(0.008)} & {\color{black}(0.081)} & {\color{black}(0.063)} & {\color{black}(0.044)} & {\color{black}(0.053)} & {\color{black}(0.065)}  & {\color{black}(0.025)} & {\color{black}(0.028)} & {\color{black}(0.015)} & {\color{black}(0.006)} & {\color{black}(0.000)} & {\color{black}(0.000)} & {\color{black}(0.000)} & {\color{black}(0.000)} \\ 
                    & \multirow{2}{*}{F1} & 0.705 & 0.811 & 0.766 & 0.770 & 0.996 & 0.212 & 0.709 & 0.788 & 0.798 & 0.775    & 0.712 & {\color{black}0.627} &  {\color{black}0.958} & {\color{black}0.997} &   \underline{\textbf{1.000}} &  \underline{\textbf{1.000}} & \underline{\textbf{1.000}} & \underline{\textbf{1.000}} \\
\hline
\end{tabular}}
\end{table*}

{\color{black}Tables~\ref{table3} and ~\ref{table4} summarize the performance of the proposed methods and the state-of-art frameworks for classifying the imbalanced and balanced datasets based on the AUC-ROC and F1 metrics, respectively. These tables represent the mean (standard deviation) values of the accuracy measures obtained through 10-fold cross-validation. For the imbalanced datasets, the proposed BA-SkewPNN model achieves the highest AUC-ROC values in 14 out of 20 cases and the best F1 scores in 9 datasets. Similarly, the SkewPNN framework demonstrates excellent performance over state-of-the-art classifiers across several real-world datasets, as evident from its AUC-ROC and F1 metrics. Among the competing models, HNet achieves the highest classification accuracy for the page-block and Ecoli-0-2-3-4 vs 5 datasets. For the imbalanced multiclass datasets, the proposed approaches either outperform or perform comparably to the baseline frameworks. Specifically, for the dermatology-6 and flare-F datasets, the SkewPNN, BA-SkewPNN, PNN, and BA-PNN models provide the most accurate results. The empirical evaluations on the real-world balanced datasets further underscore the superiority of the proposed BA-SkewPNN framework. In terms of the AUC-ROC metric, BA-SkewPNN achieves the highest accuracy in 6 out of 15 datasets, followed by SkewPNN, GHOST, and XGB models. Similarly, the F1 metrics reveal a consistent trend, with BA-SkewPNN excelling in 6 out of 20 datasets, followed by SkewPNN, PNN, and XGB. For the balanced multiclass datasets, the BA-SkewPNN model outperforms the state-of-the-art classifiers in 3 out of 4 cases. As outlined in Remark \ref{rem3}, here, we also experimentally observed that SkewPNN and BA-SkewPNN could not showcase the best possible performance as compared to state-of-the-art methods in Table \ref{table3} and \ref{table4} when the curse of many dimensions occurred. As the number of input features increases, the estimator of PDF becomes difficult in complex and high-dimensional spaces. Also, to some extent, large datasets contribute to the failure cases of our proposed methods in Table \ref{table3} and \ref{table4}.} Overall, the experimental results highlight the commendable performance of the proposed SkewPNN and BA-SkewPNN models for real-world classification tasks. This evaluation also provides valuable insights into the strengths and limitations of different state-of-the-art classifiers. In essence, the proposed methods consistently demonstrate superior performance in terms of accuracy measures and excel both in imbalanced and balanced dataset scenarios.

\subsection{Statistical Tests for Comparison of Classifiers}\label{Sec_stat_signif}
{\color{black}We employ several non-parametric statistical tests to assess the robustness of various classification approaches. Among these, the multiple comparison with the best (MCB) test is used as a post-hoc statistical procedure to determine the `best' classifier and identify significant performance differences between competing approaches and the `best' model \cite{koning2005m3}. The MCB test ranks models based on their accuracies across different classification tasks, identifying the model with the lowest average rank as the `best-performing' approach. It then calculates the critical distance (CD) for all models and compares the CDs of competing approaches against the `best-performing' framework. Fig. \ref{fig:skewpnn_mcb} presents the MCB test results for imbalanced and balanced classification tasks, evaluated using AUC-ROC and F1 metrics. For the imbalanced classification tasks, the proposed BA-SkewPNN approach achieves the lowest average ranks of 2.25 (w.r.t. AUC-ROC) and 4.22 (w.r.t. F1), followed by the BA-PNN and SkewPNN models. Similarly, for balanced classification tasks, the BA-SkewPNN approach records the lowest ranks of 4.37 (w.r.t. AUC-ROC) and 6.13 (w.r.t. F1), followed by the GHOST and XGB frameworks. Thus, the BA-SkewPNN model emerges as the `best-performing' classifier across all scenarios. Its CD, highlighted in the shaded region, serves as the reference value for the test. Since the CDs of most competing models do not overlap with the reference value, we conclude that their performance is significantly worse than the `best-performing' BA-SkewPNN model.  

Alongside the MCB testing procedure, we employ the non-parametric Friedman test to detect statistical differences in model performances \cite{friedman1937use}. This distribution-free approach evaluates the null hypothesis that the classification performance of all the competing approaches is equivalent based on their average ranks. The test rejects the null hypothesis of statistical equivalence if the computed p-value is less than the specified significance level. We conduct the Friedman test for both imbalanced and balanced classification tasks using the AUC-ROC and F1 metrics at a 5\% significance level. For the imbalanced classification tasks, the Friedman test yields p-values of $2.2e^{-16}$ (AUC-ROC) and 0.00001 (F1 metric). For the balanced classification task, the corresponding p-values are 0.0001 (AUC-ROC) and 0.00002 (F1 metric). Since all the calculated p-values are significantly below the threshold of 0.05, we reject the null hypothesis. This indicates that the classification performances of various models differ significantly across both imbalanced and balanced datasets. 

Furthermore, to identify which models differ significantly in classifying the imbalanced and balanced datasets, we conduct a post-hoc analysis using the Wilcoxon signed-rank test \cite{woolson2007wilcoxon}. This non-parametric statistical procedure conducts pairwise comparisons between individual classifiers, testing the null hypothesis that no significant difference exists between their performances. In our analysis, the Wilcoxon signed-rank test is applied to determine which competing models differ significantly from the proposed SkewPNN and BA-SkewPNN framework at a 5\% level of significance. Table \ref{tab:stat_sig} summarizes the p-values obtained by comparing the AUC-ROC and F1 metrics of the SkewPNN and BA-SkewPNN models with the competing classifiers for imbalanced and balanced datasets. For the imbalanced classification tasks, the results show that the BA-SkewPNN framework exhibits significant performance differences from all the competing models across both metrics, except for the GHOST model in terms of the F1 metric. Similarly, the SkewPNN model demonstrates significant performance differences from a bunch of competing methods, as determined by both AUC-ROC and F1 metrics. In balanced classification tasks, the p-values indicate that significant performance differences exist between the proposed models and most competing approaches. Overall, the statistical significance tests support the conclusion that the BA-SkewPNN and SkewPNN models exhibit distinct and superior performance compared to other methods for the majority of datasets.} 

\begin{figure}
    \centering
    \includegraphics[width=0.95\linewidth]{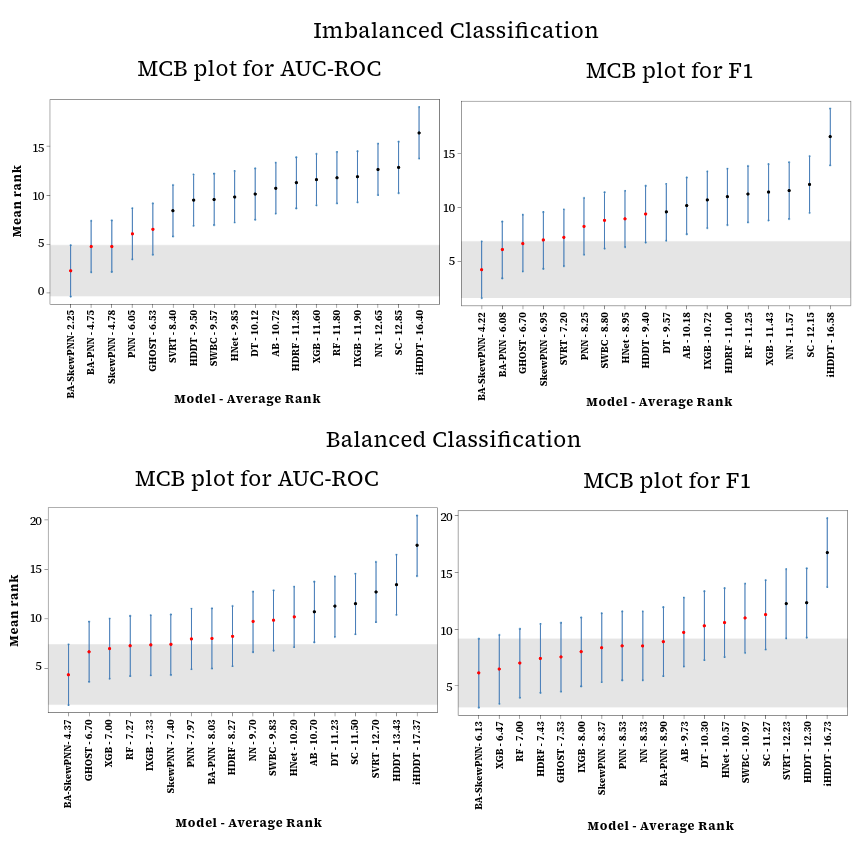}
    \caption{{\color{black}MCB test results for imbalanced and balanced classification tasks were evaluated based on the AUC-ROC and F1 metrics. In the plot, `BA-SkewPNN - 2.25' represents the average rank of the BA-SkewPNN framework (2.25) as computed using the AUC-ROC metric for imbalanced datasets. Similar interpretations apply to other methods.}}
    \label{fig:skewpnn_mcb}
\end{figure}


\begin{table}[!ht]
    \centering
    \caption{{\color{black}Wilcoxon signed-rank test p-values comparing the performance of SkewPNN and BA-SkewPNN with state-of-the-art methods on both imbalanced and balanced datasets. An asterisk (*) indicates significant performance differences between the proposed models and competing methods at the 5\% significance level.}}
    \scriptsize
    \label{tab:stat_sig}
    \begin{tabular}{|c|c|c|c|c|c|c|c|c|}
    \hline
        \multirow{3}{*}{Models}& \multicolumn{4}{c|}{Imbalanced Classification} & \multicolumn{4}{c|}{Balanced Classification}\\ \cline{2-9}
        & \multicolumn{2}{c|}{SkewPNN vs. Models} & \multicolumn{2}{c|}{BA-SkewPNN vs. Models} & \multicolumn{2}{c|}{SkewPNN vs. Models} & \multicolumn{2}{c|}{BA-SkewPNN vs. Models} \\ \cline{2-9}
        & AUC-ROC & F1 & AUC-ROC & F1 & AUC-ROC & F1 & AUC-ROC & F1 \\  \hline
        DT & 0.0001* & 0.0947 & 0.0001* & 0.0060* & 0.1817 & 0.6192 & 0.0151 & 0.1384 \\ 
        NN & 0.0002* & 0.0141* & 0.0001* & 0.0005* & 0.2444 & 0.2271 & 0.0252* & 0.0442* \\ 
        AB & 0.0001* & 0.0606 & 0.0000* & 0.0008* & 0.1147 & 0.3808 & 0.0002* & 0.0660 \\ 
        RF & 0.0003* & 0.0191* & 0.0001* & 0.0026* & 0.5980 & 0.7378 & 0.0483* & 0.3193 \\ 
        SC & 0.0001* & 0.0120* & 0.0000* & 0.0004* & 0.1262 & 0.3350 & 0.0021* & 0.0416* \\ 
        SVRT & 0.0004* & 0.2729 & 0.0000* & 0.0220* & 0.0677 & 0.1514 & 0.0001* & 0.0027* \\
        HDDT & 0.0002* & 0.0825 & 0.0001* & 0.0148* & 0.0416* & 0.2622 & 0.0013* & 0.0151* \\
        HDRF & 0.0003* & 0.0120* & 0.0000* & 0.0021* & 0.2271 & 0.3882 & 0.0277* & 0.1147 \\ 
        XGB & 0.0000* & 0.0068* & 0.0000* & 0.0001* & 0.5788 & 0.7349 & 0.1514 & 0.3560 \\ 
        IXGB & 0.0004* & 0.0448* & 0.0001* & 0.0024* & 0.5980 & 0.6401 & 0.0527 & 0.2216 \\ 
        HNet & 0.0023* & 0.2283 & 0.0002* & 0.0173* & 0.1600 & 0.4020 & 0.0042* & 0.1006 \\ 
        iHDDT & 0.0000* & 0.0000* & 0.0000* & 0.0000* & 0.0000* & 0.0001* & 0.0004* & 0.0002* \\ 
        GHOST & 0.0220* & 0.7392 & 0.0001* & 0.0836 & 0.4750 & 0.7523 & 0.2807 & 0.5227 \\ 
        SWBC & 0.0002* & 0.1841 & 0.0000* & 0.0220* & 0.2622 & 0.3808 & 0.0075* & 0.0319* \\ 
        PNN & 0.0021* & 0.0537 & 0.0003* & 0.0003* & 0.2778 & 0.2376 & 0.0298* & 0.0513 \\ 
        BA-PNN & 0.8912 & 0.9723 & 0.0001* & 0.0173* & 0.4009 & 0.5000 & 0.0008* & 0.0101* \\ 
        SkewPNN & ~ & ~ & 0.0013* & 0.0010* & ~ & ~ & 0.0394* & 0.0450* \\ 
        BA-SkewPNN & 0.9989 & 0.9992 & ~ & ~ & 0.9657 & 0.9606 & ~ & ~ \\ \hline
    \end{tabular}
\end{table}

\section{Conclusion and Discussion}\label{sec5}
Imbalanced datasets pose complexities in machine learning problems, particularly in classification endeavors. This disparity can prompt models to favor the over-represented class, resulting in compromised performance of the minority class. This discrepancy becomes particularly crucial when the minority class holds pivotal insights or embodies infrequent yet pivotal occurrences. In classification scenarios, when class distribution is notably skewed, with one class holding notably more samples (majority class) than others (minority class), the dataset is labeled as imbalanced. This study introduced SkewPNN and BA-SkewPNN algorithms leveraging PNN in tandem with the skew-normal kernel to address the challenge associated with imbalanced datasets. Through this integration, we harness the PNN's unique capability to provide probabilistic outputs, allowing for a nuanced understanding of prediction confidence and adept handling of uncertainty. Incorporating the skew-normal kernel, renowned for its adaptability in handling non-symmetric data, markedly enhances the representation of underlying class densities. To optimize performance on imbalanced datasets, fine-tuning hyperparameters is crucial, a task we undertake using the Bat optimization algorithm. Our study demonstrates the statistical consistency of density estimates, indicating a smooth convergence to the true distribution with increasing sample size. {\color{black}We have also outlined the theoretical analysis of computational complexity for our proposed algorithms.} 
Furthermore, extensive simulations and real-world data examples comparing various machine learning models affirm the SkewPNN and BA-SkewPNN effectiveness on both balanced and imbalanced data. 

{\color{black}SkewPNN can handle data types that are commonly used in classification tasks. Continuous features are inherently well-suited to the skew-normal kernel due to its flexibility in capturing skewness. For categorical features, one-hot encoding is used and SkewPNN is compatible with these transformations. SkewPNN's ability to capture asymmetrical feature distributions makes it robust for imbalanced data, confirmed by our experimental results and its generalizability across diverse domains. The key limitation of the proposed framework is the scalability issue for large and high-dimensional datasets due to its reliance on storing training instances. This is why our proposal may require more memory for big data problems, and performance may degrade as observed in experimental settings. It may be handled by incorporating a nearest neighbor-based approach, which will determine how many data points from the training set will be considered when calculating distances in the pattern layer of SkewPNN and BA-SkewPNN. A trade-off between the number of nearest neighbors, computational costs, and accuracy will result in a more robust algorithm for big data problems -- this can be considered as an immediate future research avenue of this paper.} Looking ahead, other future research could also focus on working on imbalanced regression problems~\cite{yang2021delving} where hard boundaries between classes do not exist. We can extend our current work to accommodate regression tasks that involve continuous and even infinite target values (e.g., the age of different people based on their visual appearances in computer vision, where age is a continuous target and can be highly imbalanced).

\section*{CRediT authorship contribution statement}
{\bf Shraddha M. Naik}: Data curation; Formal analysis; Investigation; Visualization; Software; Writing - original draft. 
{\bf Tanujit Chakraborty}: Conceptualization; Methodology; Formal analysis; Investigation; Validation; Writing - original draft; Writing - review \& editing.
{\bf Madhurima Panja}: Formal analysis; Investigation; Visualization; Writing - review \& editing.
{\bf Abdenour Hadid}: Supervision; Project administration; Writing - review \& editing.
{\bf Bibhas Chakraborty}: Supervision; Project administration; Writing - review \& editing.

\section*{Declaration of competing interest}
The authors declare that they have no known competing financial interests or personal relationships that could have appeared to influence the work reported in this paper.

\section*{Acknowledgments}
We thank the editor, the associate editor, and three reviewers for their insightful comments and constructive feedback.

{\color{black}
\section{Appendix}\label{appendix}
\noindent {\bf Illustrative Example with Numerical values:}\\
\noindent We provide an illustrative example with numerical values demonstrating how the SkewPNN algorithm works for imbalanced datasets. For simplicity and easy computation, we consider the training sample of size 10, and Table \ref{tab:numerical_example} displays this example dataset with IR = 4. 
\begin{itemize}
    \item Class 0 (Majority Class): Clustered around [2,2].
    \item Class 1 (Minority Class): Clustered around [4,4].
    \item Our objective is to classify a test point [3.5,3.5].
\end{itemize}

\begin{table}[!ht]
\caption{{\color{black}A toy data example with numerical values having ten observations (eight class 0 examples and two class 1 examples).}}
    \label{tab:numerical_example}
    \centering
    \begin{tabular}{|c|c|c|c|c|c|c|c|c|c|c|}
\hline Feature 1 & 2.0 & 2.2 & 2.4 & 2.6 & 2.8 & 1.8 & 1.9 & 2.3 & 4.0 & 4.1 \\ \hline
       Feature 2 & 2.0 & 2.2 & 2.4 & 2.6 & 2.8 & 1.8 & 2.1 & 2.1 & 4.0 & 4.1 \\ \hline
       Class     & 0   & 0   & 0   &  0  &  0  & 0   & 0   & 0   & 1   & 1 \\ \hline
\end{tabular}
\end{table}

\noindent Given the input data, the pattern layer computes the similarity between the input vector and the training samples using the Gaussian kernel function (Eqn. \ref{Eq_Gaussian_Kernel}) in PNN
\begin{equation}\label{new_eq_numerical}
    K(x,x_i) = \exp\left(-\frac{d_e^2}{2\sigma^2}\right)
\end{equation}
and skew-normal kernel function in SkewPNN (with an extra adjusting factor or constant)
\begin{equation}\label{new_eqn_numerical}
   K(x,x_i) = 2 \exp\left(-\frac{d_e^2}{2\sigma^2}\right) \Phi \left\{\alpha \left(\frac{d_e}{\sigma}\right)\right\},
\end{equation}
where $d_e$ is the Euclidean distance between the test point ([3.5,3.5]) and the training points, $\Phi$ is the CDF of standard normal distribution, $\sigma$ and $\alpha$ are the parameters controlling the scale and skewness of the kernels, respectively. It is important to note that for $\alpha=0$, SkewPNN is simply PNN with a Gaussian kernel. The summation layer aggregates the outputs of the pattern layer neurons for each class. Finally, the output layer provides the probability of the input vector (test sample here) belonging to each class, and the class with the highest probability is chosen as the predicted class. This can be done in two ways: Normalization by dividing the kernel sum for a class by the number of training samples in each class (used in the Algorithm \ref{algo_skewpnn}) or normalizing by dividing the kernel sum for a class by the total sum across all classes. When classes are equal in size, the output of the two normalization methods matches. The normalization by total kernel values produces true posterior probabilities and ensures probabilities sum to 1; therefore, we adopt it for this numerical example for its simplicity ($c$ denotes class label):
$$
P(c \mid x)=\frac{\sum_{i \in c} K\left(x, x_i\right)}{\sum_j \sum_{i \in c_j} K\left(x, x_i\right)}.
$$
We work with this numerical data to provide a detailed calculation for the predicted class using PNN and SkewPNN. The computational results are depicted in Table \ref{tab:results_numerical}, and we plotted the decision boundaries generated by both the models in Fig. \ref{plot_decision_boundary_numerical}. 

\begin{table}[!ht]
\caption{{\color{black}Values of different layers of PNN and SkewPNN calculated for the toy example. The networks assign the test input to the class with the highest estimated probability (highlighted in bold).}}
    \label{tab:results_numerical}
    \centering
\resizebox{1.0\textwidth}{!}{\begin{tabular}{|l|l|l|l|l|}
\hline \multirow{2}{*}{${x_i}$} & \multirow{2}{*}{$d=||{x_i}-{x}||$} & Gaussian Kernel & Skew-Normal Kernel with & \multirow{2}{*}{Class Labels} \\ 
 &  & with $\sigma =1$ (Eqn. \ref{new_eq_numerical}) & $\sigma =1, \; \alpha=-2$ (Eqn. \ref{new_eqn_numerical}) & \\
\hline [2.0, 2.0] & 2.1213 & 0.1053 & $2.3283E-06$ & 0 \\ \hline
       [2.2, 2.2] & 1.8384 & 0.1845 & $4.3552E-05$ & 0 \\ \hline
       [2.4, 2.4] & 1.5556 & 0.2981 & 0.0005       & 0 \\ \hline
       [2.6, 2.6] & 1.2727 & 0.4448 & 0.0048       & 0 \\ \hline
       [2.8, 2.8] & 0.9899 & 0.6126 & 0.0292       & 0 \\ \hline 
       [1.8, 1.8] & 2.4041 & 0.0555 & $8.4586E-08$ & 0 \\ \hline
       [1.9, 2.1] & 2.1260 & 0.1043 & $2.2102E-06$ & 0 \\ \hline
       [2.3, 2.1] & 1.8439 & 0.1826 & $4.1320E-05$ & 0 \\ \hline
       [4.0, 4.0] & 0.7071 & 0.7788 & 0.1225       & 1 \\ \hline
       [4.1, 4.1] & 0.8485 & 0.6976 & 0.0625       & 1 \\ \hline
       \multirow{3}{*}{Summation} & sum(kernel values & \multirow{2}{*}{1.9882}   & \multirow{2}{*}{0.0347}  & \multirow{3}{*}{Predicted class}\\ 
       & for class 0) & & & \\ \cline{2-4}      
       layer & sum(kernel values & \multirow{2}{*}{1.4765}  &  \multirow{2}{*}{0.1850} &  for test point (below) \\ 
        & for class 1) & & & \\ \hline    
       \multirow{3}{*}{Output} & \multirow{2}{*}{$\operatorname{Prob}\left(\operatorname{class} 0 \; | \; x\right)$} & \multirow{2}{*}{$\frac{1.9882}{(1.9882+1.4765)}$={\bf 0.5738}}   & \multirow{2}{*}{0.1580}  & \multirow{2}{*}{PNN predicts class 0}\\ 
       & & & & \\ \cline{2-5}      
       layer & \multirow{2}{*}{$\operatorname{Prob}\left(\operatorname{class} 1 \; | \; x\right)$} & \multirow{2}{*}{0.4262}   & \multirow{2}{*}{$\frac{0.1850}{(0.0347+0.1850)}$=\bf 0.8420}  & \multirow{2}{*}{SkewPNN predicts class 1}\\ 
       & & & & \\  \hline
\end{tabular}}
\end{table}

\begin{figure*}[htbp]
     \centering
     \begin{subfigure}[b]{0.49\textwidth}
         \centering
         \includegraphics[width=8cm, height=6.5cm]{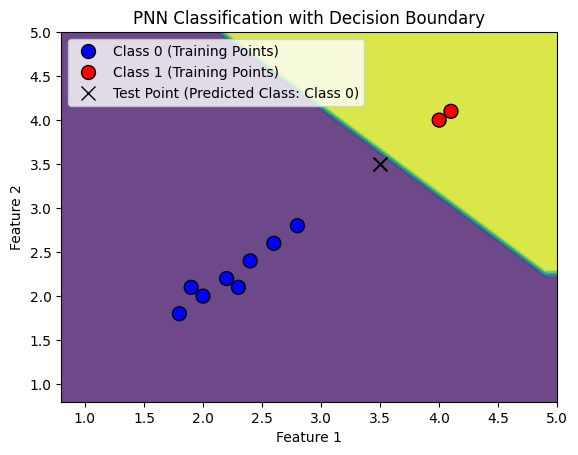}
         \caption{ }
         \label{fig:1}
     \end{subfigure}
     \hfill
     \begin{subfigure}[b]{0.49\textwidth}
         \centering
         \includegraphics[width=8cm, height=6.5cm]{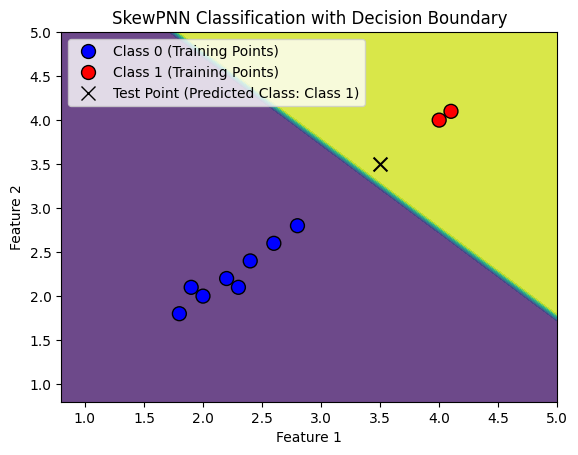}
         \caption{ }
     \end{subfigure}
        \caption{{\color{black}Plots of decision boundaries separating two classes (Class 0 and 1) using (a) PNN algorithm and (b) SkewPNN algorithm. The plot also depicts the predicted class label for the test point.}}
        \label{plot_decision_boundary_numerical}
\end{figure*}

\noindent This illustrative example with numerical values finds that if we use the Gaussian kernel, the majority of class points dominate the probability distribution, leading to the classification of the test point as class 0 (the minority class points have a much smaller contribution). Now, instead of a Gaussian kernel, if we use a skew-normal kernel with the skewness parameter (here we choose $\alpha = -2$), it amplifies the influence of the minority class points, making their probability contribution more significant. This leads to a different classification result, favoring class 1 in this example, as observed in Table \ref{tab:results_numerical} and Fig. \ref{plot_decision_boundary_numerical}. This example demonstrates how the SkewPNN works for imbalanced binary classification datasets.}

\section*{Code and Data Availability}
To support reproducible research, we are making the code and the data publicly accessible at \url{https://github.com/7shraddha/SkewPNN}.

\begin{scriptsize}
\bibliographystyle{elsarticle-harv} 
\bibliography{example}
\end{scriptsize}

\end{document}